\documentclass[11pt, a4paper, logo, copyright]{googledeepmind}

\usepackage[authoryear, sort&compress, round]{natbib} 
\usepackage{graphicx}      
\usepackage{subcaption}    
\usepackage{geometry}      
\usepackage{amsmath}       
\usepackage{xcolor}        
\usepackage{siunitx}  
\usepackage{float}
\usepackage{wrapfig}
\usepackage{geometry}
\usepackage{amsmath,amssymb,mathtools}
\usepackage{bbm}
\usepackage{caption}
\usepackage{dsfont}  
\usepackage{booktabs}
\usepackage{enumitem}
\usepackage{algorithm}
\usepackage{caption}
\usepackage{algpseudocode}
\usepackage{hyperref}
\usepackage{xcolor}

\newtheorem{claim}{Claim}

\newenvironment{designrationale}[1]{%
  \begin{tcolorbox}[
    colback=gray!5!white,
    colframe=black!75!black,
    fonttitle=\bfseries,
    title={#1},
    boxrule=0.5pt,
    arc=2pt,
    left=6pt,
    right=6pt,
    top=6pt,
    bottom=6pt
  ]%
}{%
  \end{tcolorbox}%
}

\hypersetup{
  colorlinks=true,
  linkcolor=black,
  citecolor=black,
  urlcolor=blue,
  pdftitle={A Subgoal-driven Framework for Improving Long-Horizon LLM Agents}
}

\usepackage{makecell}
\usepackage{enumitem}
\usepackage{threeparttable} 
\usepackage{booktabs}      
\usepackage{array}         
\usepackage{longtable}     
\usepackage{enumitem}      
\usepackage{tabularx}      

\definecolor{highlightcolor}{RGB}{235, 245, 255} 
\definecolor{ourmodelcolor}{RGB}{230, 240, 255}

\usepackage{multirow}   
\usepackage{colortbl}   
\usepackage{xcolor}     
\usepackage{caption}    
\usepackage{array}      
\usepackage{algorithm}
\usepackage{algpseudocode}

\usepackage{listings}      
\usepackage[most,skins,theorems]{tcolorbox} 

\definecolor{darkpurple}{RGB}{102, 51, 153}
\definecolor{lightpurple}{RGB}{204, 153, 255}
\definecolor{lightblue}{rgb}{0.22,0.45,0.70}
\definecolor{forestgreen}{rgb}{0.24,0.50,0.19}
\definecolor{DeepTeal}{RGB}{20, 110, 110}
\definecolor{LightTeal}{RGB}{235, 245, 245}

\newcolumntype{L}[1]{>{\raggedright\arraybackslash}p{#1}}

\tcbset{
  aibox/.style={
    width=\linewidth,
    top=8pt,
    bottom=4pt,
    colback=LightTeal,
    colframe=DeepTeal,
    colbacktitle=DeepTeal,
    fonttitle=\bfseries,
    enhanced,
    center,
    attach boxed title to top left={yshift=-0.1in,xshift=0.15in},
    boxed title style={boxrule=0pt,colframe=white,},
  }
}
\newtcolorbox{AIbox}[2][]{aibox,title=#2,#1}

\title{A Subgoal-driven Framework for Improving Long-Horizon LLM Agents}
\correspondingauthor{taiyiw@google.com}

\author[1]{Taiyi Wang}
\author[1]{Sian Gooding}
\author[1]{Florian Hartmann}
\author[1]{Oriana Riva}
\author[1]{Edward Grefenstette}

\affil[1]{Google DeepMind}

\begin{abstract}
Large language model (LLM)-based agents have emerged as powerful autonomous controllers for digital environments, spanning mobile interfaces, operating systems, and web browsers. Web navigation, for example, demands handling dynamic content and long action sequences, making it a particularly complex task. Existing LLM-backed agents exhibit weakened long-horizon planning abilities on two fronts. During online execution, agents often lose track as new information arrives, lacking a clear and adaptive path toward the final goal. This issue is further compounded during RL fine-tuning, where sparse and delayed rewards make it difficult for agents to identify the actions that lead to success, preventing them from sustaining coherent reasoning over extended tasks. We address this with two contributions: (1) An agent framework leveraging proprietary models for online planning via subgoal decomposition; and (2) MiRA (\underline{Mi}lestoning your \underline{R}einforcement Learning Enhanced \underline{A}gent), an RL training framework using dense, milestone-based reward signals. Real-time planning mechanism enhances proprietary models like Gemini by $\approx$10\% absolute success rate (SR) on the WebArena-Lite benchmark. Meanwhile, applying MiRA to the open \textbf{Gemma3-12B} model boosts its success rate from $6.4\%$ to $43.0\%$. This performance surpasses proprietary systems such as GPT-4-Turbo ($17.6\%$) and GPT-4o ($13.9\%$), as well as the previous open-model state of the art, \textbf{WebRL} ($38.4\%$). Our findings demonstrate that combining explicit inference-time planning with milestone-based rewards significantly boosts an agent's long-horizon abilities, paving the way for more robust, general-purpose autonomous systems.

\end{abstract}

\begin{document}

\maketitle


\section{Introduction}

\begin{figure}[htbp]
    \centering
    \includegraphics[width=1.0\linewidth]{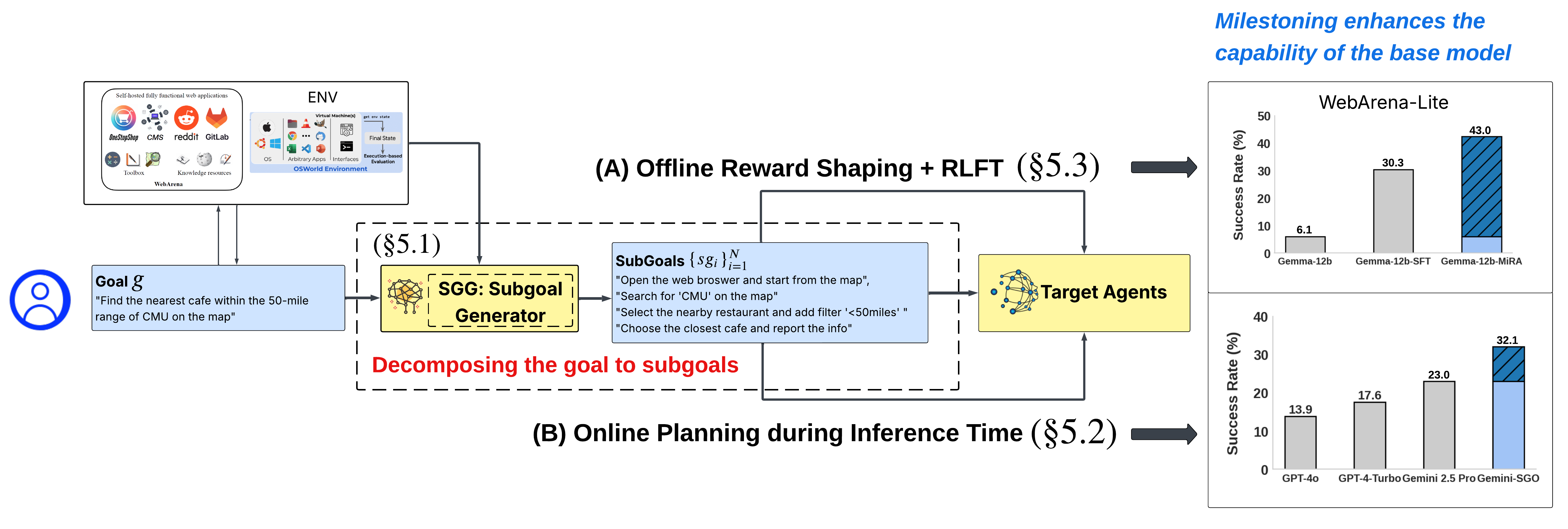}
    
    \captionsetup{justification=centering}
    
    \caption{Overview of the Milestoning the Agents}
    \label{fig:Intro}
\end{figure}


Large language model (LLM)-based agents have gained significant traction as autonomous interfaces for navigating and interacting with real-world digital environments. These systems span a broad spectrum of modalities, ranging from mobile device automation~\citep{yang2023appagent,wang2024mobile,christianos2024lightweight,papoudakis2025appvlm,li2024effects,rawles2024androidinthewild,rawles2024androidworld} and operating system (OS) control~\citep{xie2024osworld,hu2025agents,zhang2025ufo} to complex web navigation~\citep{wang2024distrl,qi2024webrl,gu2025mobile,deng2023mind2web}. Among these, \textbf{web navigation} serves as a particularly rigorous testbed for evaluating an agent's reasoning capabilities due to the inherent complexity of the environment. The difficulty of this domain is underscored by recent proprietary evaluations: while the state-of-the-art Gemini 2.5 Computer Use model achieves a 75\% success rate on aggregate UI control tasks, its performance drops significantly to 36\% on open-ended benchmark like \emph{WebWorld}~\citep{zhang2024large,GoogleDeepMind_Gemini2.5_2025}. This discrepancy highlights that while large proprietary backbones can leverage their inherent reasoning for general UI tasks, complex web interaction remains a distinct unsolved challenge. In contrast, smaller or open-source models typically rely on fine-tuning to bridge the capability gap. \emph{Supervised fine-tuning} (SFT) on human or synthetic demonstrations is widely adopted but remains limited by static data and poor generalization. By comparison, \textbf{reinforcement learning (RL)} provides a more adaptive framework for long-horizon optimization, especially when guided by \emph{dense, milestone-based feedback}~\citep{bai2024digirl,wang2024distrl}.



To rigorously evaluate the capabilities of aforementioned LLM-based web agents, a growing set of realistic web interaction benchmarks has emerged, including Mind2Web~\citep{deng2023mind2web}, WebArena~\citep{zhou2023webarena}, and WebShop~\citep{yao2022webshop}, as well as interactive simulation suites such as WebGames~\citep{thomas2025webgames}. These environments simulate realistic browsing and interaction scenarios across diverse domains—ranging from e-commerce and productivity tools to social platforms and general website control or navigation tasks—requiring agents to ground language understanding in multi-step decision-making. Each benchmark poses unique challenges in perception, action planning, and context maintenance, collectively providing a comprehensive testbed for studying web-based autonomy. However, even with these advancements in environment design, current agents still exhibit significant weaknesses on long-horizon tasks~\citep{zhou2023webarena,zhang2024large}: proprietary systems like Gemini Agent and open-source models alike experience sharp performance degradation as task complexity and sequence length increase. This trend underscores that real-world web interaction remains a long-standing challenge for sustained reasoning and adaptive planning across many steps.



Our empirical analysis further quantifies this failure: it stems from a breakdown in effective real-time planning and some overly distant, unreasonable goal settings. In other words, agents frequently enter non-productive action loops or commit to suboptimal goal paths, failing to identify the next logical milestone that would lead to progress. This problem persists across model scales and training paradigms. On the WebArena-Lite benchmark~\citep{liu2024visualagentbench,koh2024visualwebarena}, agents equipped with out-of-the-box proprietary models such as Gemini-2.5-Pro exhibit ``mid-task stuck'' behaviors in nearly 50\% of evaluation trajectories. Even after supervised fine-tuning (SFT) on human demonstrations, smaller open models like Gemma-12B-SFT still fail to progress in over 30\% of cases. The same deficiency dominates error distributions in previous specialized state-of-the-art open agentic systems such as WebRL(Llama3-8B), where high-level goal instability—often resulting from the divergence towards the target objective. Collectively, these observations highlight that current models—regardless of scale or fine-tuning—lack the robust internal planning and milestone-awareness required to sustain reasoning over extended interactions.



While prior approaches have attempted to handle complex agentic tasks through interleaved reasoning traces \citep{yao2022react}, static decomposition strategies \citep{zhou2022least}, or tree-based search \citep{yao2023tree}, a promising direction to mitigate their limitations is to introduce \textbf{explicit milestones or subgoals}, which serve as intermediate checkpoints guiding the agent's progress toward the final objective. To this end, two primary research directions have emerged. Hierarchical approaches like VSC-RL \citep{wu2025vsc} and stepstone curricula \citep{sharma2021autonomous} decompose tasks using intermediate objectives, yet often rely on latent representations and brittle RL formulations that suffer from severe training instability. Conversely, Process Reward Models (PRMs)~\citep{cui2025process,xi2025agentprm} provide dense feedback but typically depend on soft, learned signals susceptible to noise and over-optimization. Our method distinguishes itself by synthesizing these paradigms by unifying explicit, coarse-grained milestoning across both inference and training. These milestones or subgoals can be incorporated either at inference time—via structured planning and dynamic decomposition—or during reinforcement learning (RL) fine-tuning through milestone-based reward shaping. However, integrating such \emph{milestoning} into web agents introduces three key challenges: (C.1) \emph{Where do subgoals come from, and how reliable are they?} (C.2) \emph{How can subgoal reasoning be integrated at inference time without prohibitive latency or contextual overhead?} and (C.3) \emph{How can intermediate rewards be embedded in RL training to improve credit assignment and stability without hindering final goal completion?}



To address these challenges, we propose a \textbf{subgoal-assisted framework} that unifies \emph{online inference-time planning} with \emph{offline RL fine-tuning} via milestone-based shaping. As illustrated in Figure~\ref{fig:Intro}, our system decomposes high-level goals into structured subgoals, enabling the agent to reason hierarchically during inference while receiving denser feedback during training.  Fundamentally, our approach follows a simple principle: \emph{``If the final goal is difficult to reach directly, increasing the probability of reaching meaningful intermediate milestones helps.''} 
Our main contributions are summarized as follows:

\begin{enumerate}
    \item \textbf{Automated Failure Analysis:} We introduce an automated failure analyzer that systematically uncovers the dominant failure modes in Web-Navigation tasks. This analysis reveals key planning and complexity bottlenecks in existing agents and directly motivates our framework design.

    \item \textbf{Inference-Time Planning with Subgoals:} We integrate lightweight subgoal-guided planning directly into the agent’s inference loop, improving long-horizon reasoning and execution for both open and proprietary LLM backbones.

    \item \textbf{Milestone-Based Offline RL Fine-Tuning (MiRA):} We develop a complementary offline RL procedure that uses milestone-driven reward shaping to provide denser training signals, effectively mitigating the sparse-reward challenges inherent in Web-Navigation.
\end{enumerate}

\section{Related Work}

The creation of autonomous Computer Use Agents for complex online environments lies at the intersection of perception, planning, and control. The recent rise of large language models (LLMs) has provided a strong foundation for reasoning and decision making, enabling agents to operate in realistic, text-rich domains such as OS control and the web navigation \citep{koh2024visualwebarena,zhang2025ufo}. 


\subsection{GUI Agent and Reinforcement Learning Fine-Tuning}


LLM-based agents for GUI control or more specifically computer usage can be grouped into three paradigms: \emph{prompting-based}, \emph{imitation-based}, and \emph{RL-based}. Prompting-based agents steer frozen foundation models through structured instructions and tool use, achieving strong zero-shot performance but limited adaptability. Imitation-based agents instead rely on supervised fine-tuning (SFT) over human or synthetic demonstrations, which improves task alignment but depends on static data and fails to teach recovery from errors~\citep{zhai2024fine}. This brittleness has motivated a shift toward reinforcement fine-tuning, where agents learn from active interaction rather than passive replay. However, even with RL, adapting to complex workflows remains difficult due to the sparsity of feedback signals in multi-step environments.


This challenge of sustained reasoning in long-horizon tasks is not unique to a single domain, i.e., it remains a persistent bottleneck across the entire spectrum of GUI-based agents, including mobile device controllers \citep{bai2024digirl, bai2025digi, yang2023appagent, hong2312cogagent,wang2024distrl}, operating system assistants \citep{zhang2025ufo, xie2024osworld}, and general desktop automation. In the specific context of web agents, recent pipelines have attempted to combine paradigms to address these limitations. For instance, the WebRL framework applies a self-evolving curriculum to transform open LLMs into competent web agents by addressing training-task scarcity, feedback sparsity, and distribution drift \citep{qi2024webrl}. Similarly, in the commercial domain, OpenAI's Operator platform demonstrates a prompting and tool-use agent capable of executing web interactions (clicking, form-filling), yet it remains limited in sustained multi-step workflows \citep{openai2025operator}. Meanwhile, IBM Research's CUGA (Configurable Generalist Agent) extends this to enterprise web/API tasks, achieving state-of-the-art results on WebArena while highlighting long-horizon failures and the need for modular, iterative improvement \citep{shlomov2025benchmarks}. Despite these advances, key challenges persist: rewards in web navigation are often binary (success/failure) after many interactions, making credit assignment difficult; consequently, RL-based agents still show steep performance drops as task length increases, indicating that robust planning and error recovery remain unresolved.

\subsection{Goal-Conditioned LLM-Agent}

While smaller LLMs serve as low-level controllers, larger models are increasingly used as high-level planners. In hierarchical settings, an LLM planner produces intermediate sub-goals that a goal-conditioned policy executes~\citep{wang2023describe,zhang2025goal}. This ``LLM-as-Guide'' design allows semantic goal decomposition and pruning of irrelevant actions. Alternatively, unified architectures treat the LLM itself as an end-to-end policy whose internal reasoning trace acts as the plan~\citep{hong2025planning}.


In web environments, however, even such planning architectures face critical limitations. For example, the CUGA error‐analysis shows that tasks exceeding $\sim$10 interaction steps often fail due to impaired sub-goal coherence, poor re-planning, and drift from the original aim~\citep{qian2025webgrapheval,marreed2025towards}. As documented in the CUGA architecture~\citep{shlomov2025benchmarks}, the system employs a hierarchical planner–executor framework with explicit task decomposition, a persistent task ledger, and reflective re-planning mechanisms that track variables, repair plans, and validate tool calls. These components provide concrete re-planning capabilities within the agent, and thus directly motivate our focus on strengthening long-horizon planning fidelity. These findings underscore that the bottleneck is not only low-level action execution but sustained planning, monitoring, and adaptation across long horizons. To address this, sub-goal generation quality, dynamic re-planning mechanisms, and low-latency inference-time planning become central design concerns—areas we build upon in our method.

Recent efforts have introduced Process Reward Models (PRMs) to mitigate these long-horizon failures~\citep{xi2025agentprm,li2024process,cui2025process,choudhury2025process}. Unlike Outcome Reward Models (ORMs) that provide sparse feedback only upon task completion, PRMs offer dense, step-by-step supervision, enabling agents to verify intermediate reasoning and correct deviations in real-time. For instance, \citet{chae2025web} proposes Web-Shepherd, which leverages checklist-style sub-goal verification to monitor web navigation trajectories, significantly reducing error propagation. Similarly, AgentPRM~\citep{xi2025agentprm} introduces a dual-scoring mechanism measuring both ``promise'' (likelihood of success) and ``progress'' (inter-step advancement), facilitating effective inference-time search and pruning of suboptimal branches. However, learned PRMs often suffer from high inference overhead and susceptibility to reward over-optimization, particularly when ground-truth intermediate supervision is scarce. To address this, our approach achieves a ``best-of-both-worlds'' balance by replacing soft rewards with hard objectives. Unlike PRMs, which rely on noisy, unverifiable scalars to estimate progress, we utilize explicit milestones as rigid semantic checkpoints. This combination allows us to retain the continuous progress tracking while ensuring the verifiable reliability of ground-truth objectives.


\subsection{Goal-Conditioned Reinforcement Learning}

Sequential web tasks naturally align with goal-conditioned reinforcement learning (GCRL) \citep{kaelbling1993learning}, where the agent optimizes cumulative reward conditioned on a task goal. A major difficulty in GCRL is \emph{reward sparsity}, which hinders efficient credit assignment. Techniques such as Hindsight Experience Replay (HER) \citep{andrychowicz2017hindsight} address this by reinterpreting failed trajectories as successes for alternative goals, yielding denser learning signals. However, standard HER assumes Markovian rewards, which often limits its applicability to the non-Markovian, long-horizon nature of web navigation.

To overcome these limitations, recent research has pivoted towards on-policy GCRL and intrinsic motivation. For instance, \citet{gong2024goal} introduces GCPO, an on-policy framework that leverages self-curriculum learning to handle non-Markovian reward structures, significantly stabilizing training in complex sequential environments. Complementing this, WebAgent-R1~\citep{wei2025webagent} demonstrates that end-to-end on-policy RL—augmented with parallel trajectory generation and ``thinking'' steps—can surpass improved imitation learning baselines without the complexity of off-policy buffer management.

Beyond pure RL updates, explicit subgoal scaffolding remains critical for exploration. While early efforts like VSC-RL~\citep{wu2025vsc} utilized subgoal-conditioned learning to boost sample efficiency, they often struggled to balance intermediate subgoal completion with final goal optimality. Newer approaches address this by learning latent subgoals or internal world models. For instance, \citet{park2023hiql} propose HIQL, which learns a high-level policy over latent states, effectively decoupling strategic planning from low-level control. In a parallel direction, \citet{duan2024learning} utilize world models to simulate next outcomes, enabling agents to explore sparse-reward environments via predicted rollouts. These methods advance the field by enabling hierarchical reasoning, allowing agents to sustain progress in long-horizon tasks through abstract, learned objectives. However, applying these latent or model-based approaches to web agents introduces critical limitations. Latent subgoals \citep{park2023hiql} lack semantic interpretability, making it impossible to explicitly verify the agent's intermediate progress. Similarly, relying on world models \citep{duan2024learning} to simulate outcomes is computationally expensive and prone to compounding errors in dynamic, open-ended web environments. In contrast, our method bypasses latent abstractions and noisy simulations entirely by grounding reasoning in explicit, semantically verifiable milestones. Furthermore, we couple this with a specialized RL fine-tuning strategy where milestones function strictly as auxiliary rewards; this guarantees that intermediate supervision stabilizes training and improves credit assignment without biasing the agent against the primary, ground-truth objective.


\section{Preliminary}


\subsection{Problem Formulation}
\label{sec:problem_formulation}

We formulate the web navigation task as a finite-horizon Partially Observable Markov Decision Process (POMDP), defined by the tuple $\mathcal{M} = \langle \mathcal{S}, \mathcal{A}, \Omega, \mathcal{T}, \mathcal{O}, R, H \rangle$. Here, $\mathcal{S}$ represents the latent environment state (e.g., server-side databases and hidden DOM elements), which is inaccessible to the agent. At each timestep $t$, the agent receives a partial observation $o_t \in \Omega$ (comprising the rendered HTML, screenshot, and task instruction $g$) governed by the observation function $\mathcal{O}(o_t \mid s_t)$.

The agent selects a discrete action $a_t \in \mathcal{A}$ (e.g., clicking, typing, scrolling) according to a policy $\pi(a_t \mid o_{\leq t})$ conditioned on the history of observations. Following the action, the environment transitions to the next state $s_{t+1}$ governed by the implicit dynamics $\mathcal{T}(s_{t+1} \mid s_t, a_t)$. The episode terminates either when the task is successfully verified or when the number of interactions reaches the horizon $H$, with the objective of maximizing the expected cumulative reward.



\noindent \textbf{State Representation.}
As claimed above, in web navigation tasks, the instantaneous observation alone is insufficient to characterize progress, because many actions depend on prior interactions (e.g., previously opened panels, typed queries, or navigation paths).
Thus, we represent the state at time $t$ as the combination of the current webpage view and the full interaction history:
\(
s_t = [a_1, a_2, \ldots, a_{t-1},\, o_t],
\)
where $o_t$ denotes the current DOM tree, textual caption, or multimodal screenshot representation. This history-augmented form captures both the evolving interface state and the agent's preceding actions, enabling more accurate modeling of long-horizon dependencies in web tasks.

\noindent \textbf{Reward and Objective.}
The environment provides a sparse binary reward signal, following standard practice in web-agent benchmarks.  
Formally, at each step the agent receives
\(
    r_t(s_t, a_t, g) = \mathbb{1}[s_{t+1} \in \mathcal{S}_g],
\)
where $\mathcal{S}_g$ denotes the set of states that satisfy the goal $g$. An episode terminates either when the goal is achieved or when the finite horizon $H$ is reached.

In practical Web Navigation setups, this reward is computed by an \emph{automatic LLM-as-Judge} that evaluates goal satisfaction using the full interaction context—not just the most recent observation. The judge takes as input the task instruction $I$, the history of actions taken so far, and the final-state information (HTML and, when available, screenshots), and then determines whether the goal condition has been met. We adopt the open-source trained ORM from prior work WebRL~\citep{qi2024webrl} as the final goal-level checker, and for the subgoal-level checker, we use the larger teacher model Gemini-2.5-pro. The prompts used are the same and are introduced in Appendix~\ref{sec:apdx_exp}.

For clarity, our notation abstracts this process into the indicator $r_t$ above. Although we later introduce a dense reward-shaping mechanism in \S\ref{sec:mira_rl}, we first present the underlying sparse objective. The learning goal is to find a policy $\pi$ that maximizes the expected discounted return
\(
J(\pi) = \mathbb{E}_{\tau \sim \pi} \Big[ \sum_{t=0}^{H} \gamma^t r_t \Big],
\)
where $\gamma \in [0,1]$ is the discount factor.


\subsection{Potential-Based Reward Shaping (PBRS)}

In sparse-reward settings such as web navigation, credit assignment becomes extremely challenging. Potential-Based Reward Shaping (PBRS)~\citep{ng2003shaping} augments rewards as
\begin{equation}
    \tilde{r}_t = r_t + \gamma \Phi(s_{t+1}) - \Phi(s_t),
\end{equation}
where the potential function $\Phi(s)$ encodes progress toward the goal. PBRS preserves the optimal policy when $\Phi$ depends only on state, and is equivalent to initializing the $Q$-function with this potential~\citep{wiewiora2003potential}. A formal equivalence proof is given in Appendix~\ref{sec:apdx_proof}.

However, exact PBRS is rarely tractable in high-dimensional environments where true potentials are unknown. As a result, practical variants relax policy-invariance guarantees in favour of learnability. Dynamic PBRS~\citep{devlin2012dynamic} allows time-varying potentials, and many modern methods instead use learned value estimates or intrinsic rewards as shaping signals~\citep{gao2015potential, burda2018exploration}. These progress-oriented signals provide dense feedback that is crucial for long-horizon tasks, even when they no longer satisfy strict PBRS assumptions. Our method follows this pragmatic shaping perspective rather than the classical PBRS setting.

\section{Motivation: Quantitative Failure Analysis and Opportunities}

Before introducing our methodology, it is essential to examine the root causes behind the persistent failures of current web agents. While prior work has primarily reported aggregate performance metrics such as task success rate or trajectory completion, these end-to-end numbers reveal little about \emph{why} agents fail. To design effective alignment strategies and curriculum improvements, we must move beyond outcome statistics and toward a fine-grained understanding of failure behaviors.

The central goal of this section is therefore to \emph{\textbf{conduct data-driven analysis towards failures behind the numbers which highly motivates our study}}---to diagnose the dominant error modes and structural weaknesses in existing large language model (LLM) web agents. Such an analysis provides the foundation for targeted improvements in planning, credit assignment, and milestone-based reasoning. To this end, we developed a comprehensive automated analyzer that performs large-scale trajectory evaluation and categorization. Leveraging the reasoning and perception capabilities of the Gemini-2.5-Flash model, our analyzer systematically inspects and summarizes tons of recorded trajectories to uncover behavioral patterns. In short, our analyzer performs three core functions, which we detail below: \textbf{(1)} it first acts as an objective judge, providing outcome-based summarization; \textbf{(2)} it then performs failure categorization into a specific error mode; and \textbf{(3)} it conducts a root-cause analysis to pinpoint the exact decision step where the agent deviated.

\subsection{Function 1: Trajectory Evaluation and Summarization}

As its first function, the analyzer provides an objective \textbf{Failure Summarization}. This process determines \textit{if} a failure occurred based on hardcoded behavioral rules before passing the trajectory to the next stage of diagnosis. This procedure gives out the failure summary: A trajectory is marked as a failure solely based on the outcome of the final benchmark checking rules (i.e., \emph{``Only judge, never guess''}). This ensures that classification is objective and free of bizarre hallucinations, focusing mostly on the key state validation (e.g., Achieved XX, not achieved XX). A typical objective summary, adhering to this rule, would be:
    \begin{tcolorbox}[
        colback=blue!5!white,    
        colframe=blue!75!black, 
        title=\textbf{Failure Summary}, 
        fonttitle=\bfseries\texttt, 
        colbacktitle=blue!20!white, 
        coltitle=black,         
        rounded corners,        
        arc=4pt,                
        boxrule=1pt             
    ]
    \texttt{``Agent navigated to Hollidaysburg (film) page. Agent extracted information from this page. Agent failed to move forward and terminated. Final benchmark check: FAILED.''}
    \end{tcolorbox}

\subsection{Function 2: Categorization via Prioritized Hardcoded Rules}

Once a trajectory is confirmed as a failure by the first function, the analyzer's second function is \textbf{Categorization via Prioritized Hardcoded Rules}. Failed trajectories are systematically classified into one of four mutually exclusive failure modes using the hard-coded rules detailed in Table \ref{tab:failure_rules}, which are applied in a strict, descending order of priority.

\begin{table}[htbp]
    \centering
    \caption{\textbf{Hardcoded Rules for Failure Mode Classification}}
    \label{tab:failure_rules}
    \begin{tabularx}{\textwidth}{>{\bfseries}p{4cm} | >{}X} 
        \toprule
        Category & Hardcoded Rules \\
        \midrule
        Stop at Wrong Page (\textit{Wrong Termination}) & The agent called the $\text{exit()}$ function or terminated, but the message was judged \textit{incorrect} by the final benchmark rules. \\
        \midrule
        Get Stuck Midway & The agent failed and exhibited repetition, characterized by identical last $N$ actions or a short action-state sequence repeating $M$ times. \\
        \midrule
        Fail to Make Reasonable Attempt & The agent already reached the right page, but didn't attempt to $\text{exit()}$ \textbf{OR} The agent failed very early, either by an immediate, short $\text{exit()}$ call or by clicking a nonsensical, distracting element as its first action. \\
        \midrule
        Others (\textit{System error, etc.}) & The agent failed, but did not trigger Rules A, B, or C, indicating a general deviation from the success path. \\
        \bottomrule
    \end{tabularx}
\end{table}

\subsection{Function 3: Identifying the Decision Step}

The third and final function of the analyzer is to pinpoint the exact \textit{Key Decision Step} where the agent deviated from an optimal or successful trajectory. This capability moves beyond simple error categorization to actionable root-cause analysis. To automate this pinpointing process with high fidelity, we employ a differential analysis strategy. The analyzer aligns the agent's failed trajectory $\tau_{\text{fail}}$ against a set of \textit{Reference Success Traces} derived from two distinct sources:
 \textbf{(1) Teacher Demonstrations:} Successful trajectories generated by superior, proprietary reasoning models (e.g., IBM-CUGA) which serve as an upper-bound oracle for reasoning quality. \textbf{(2) Golden Paths from Peers:} Validated ground-truth sequences from the training dataset or historical successful runs by other trained agents (e.g., WebRL agent). Then the analyzer scans for the \textbf{First Point of Significant Divergence}---the earliest timestep $t$ where the agent's action $a_t$ semantically contradicts the reference strategy. This method isolates the root cause by filtering out downstream errors that are merely symptoms of the initial mistake. Specifically, the system targets two primary divergence patterns:
\begin{itemize}
    \item \textbf{Semantic Deviation:} Identifying the step where the agent chooses a link or search query that conceptually drifts from the cluster of actions found in successful traces (as seen in Step 3 of Figure~\ref{fig:decision_step_example}).
    \item \textbf{Stagnation and Loop Onset:} For ``Get Stuck'' failures, the analyzer utilizes state-hashing to detect cyclic sub-graphs in the interaction history. It retroactively marks the entry point into the loop—rather than the loop itself—as the key decision step, highlighting the moment the agent failed to recognize the lack of state progression.
\end{itemize}

\begin{wrapfigure}{r}{0.6\textwidth} 
    \centering
    \vspace{-10pt}
    \includegraphics[width=0.58\textwidth]{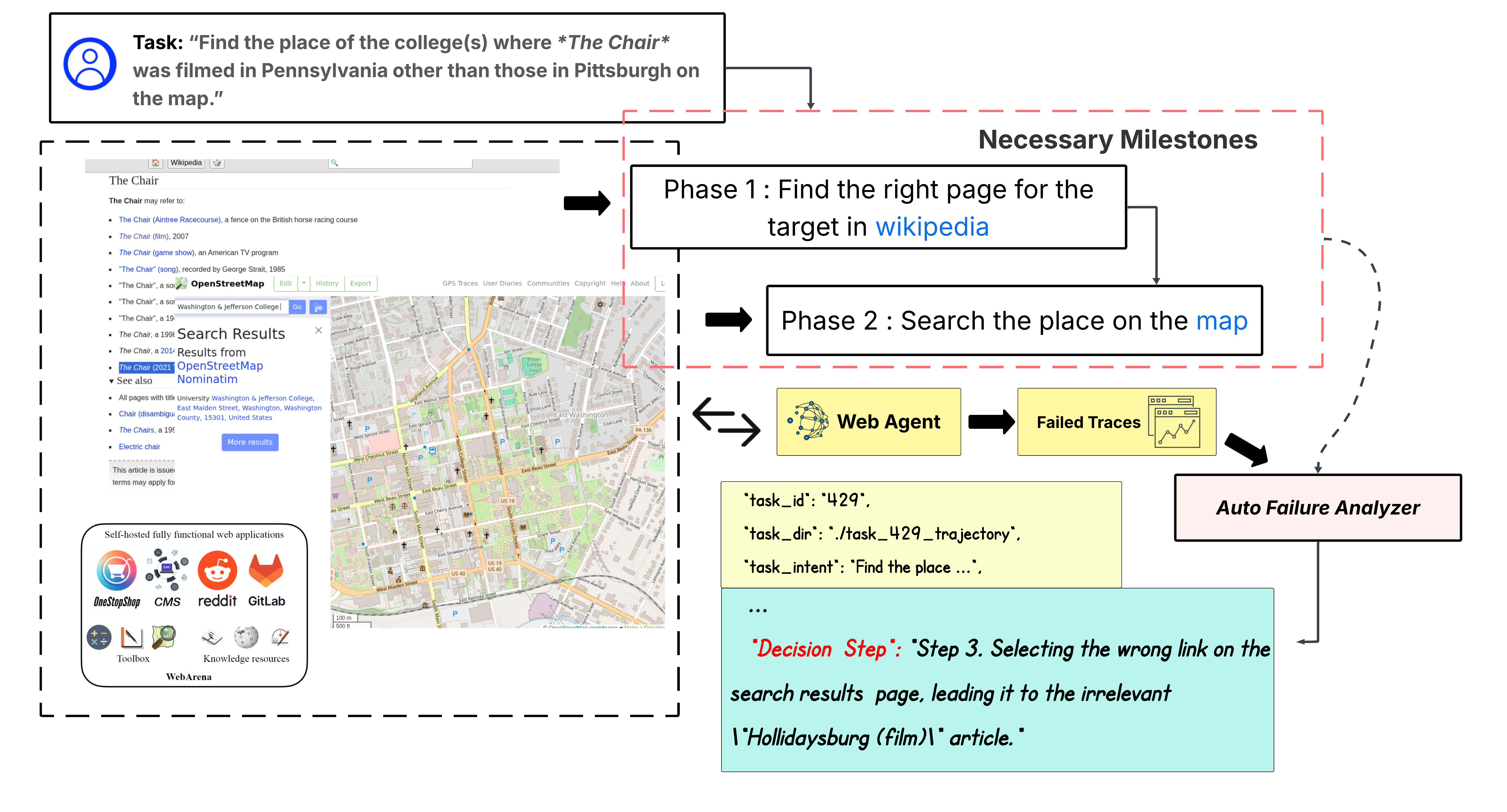} 
    \caption{\textbf{Automated Identification of Key Decision Steps.} The analyzer compares the agent's trajectory against task intent and known successful patterns to pinpoint the exact step (e.g., Step 3) where a critical deviation occurred.}
    \label{fig:decision_step_example}
    \vspace{-10pt}
\end{wrapfigure}

As illustrated in Figure \ref{fig:decision_step_example}, the analyzer parses the agent's complete trajectory against the task intent. In this typical example (Task ID 429), the goal is to find a specific filming location for ``The Chair'' in Pennsylvania. The optimal path involves two main milestones (subgoals): (1) finding the correct Wikipedia page to identify the college, and (2) searching for that college on the map.

The auto-analyzer successfully identifies that the critical failure occurred at \textbf{Step 3}, where the agent selected an incorrect link (``The Chair (2007 film)'') instead of identifying the correct TV show or other relevant entries. This single decision cascaded into a failure, as it led the agent down an irrelevant path of information.

The reliability of this identification stems directly from the robust classification rules defined previously (Function 2). Because the analyzer first rigorously classifies the \textit{type} of failure (e.g., ``Get Stuck Midway'' vs. ``Fail to Make Reasonable Attempt''), it can narrow down the search space for the decisive error. For instance, in a ``Get Stuck Midway'' scenario, the decision step is often the entry point into the repetitive loop. In a ``Wrong Termination'' scenario, it is the final action before incorrectly claiming success. By grounding the search for the key decision step in these validated failure modes, the analyzer achieves high precision in identifying the exact moment the agent's reasoning faltered.

\subsection{Soundness of the Classifier}

To ensure the reliability of our automated failure categorization, we implemented a rigorous prompting and validation framework.

\begin{wraptable}{r}{0.35\textwidth}
    \centering
    \vspace{-5pt}
    \captionsetup{font=footnotesize}
    \caption{Agreement between automated analyzer and human labels ($N=40$).}
    \label{tab:validation_scores}
    \small
    \begin{tabular}{lr}
        \toprule
        \textbf{Failure Type} & \textbf{Score} \\
        \midrule
        Stuck in Midway & 10/10 \\
        Wrong Termination & 10/10 \\
        Fail Attempt & 8/10 \\
        Others & 9/10 \\
        \bottomrule
    \end{tabular}
    \vspace{-8pt}
\end{wraptable}

\noindent \textbf{Few-Shot Prompts:} The analysis system was primed using four distinct failure examples—one for each defined category—to ground the analyzer's reasoning. This few-shot context provides the Large Language Model with concrete reference points for distinguishing between subtle failure modes (e.g., distinguishing a ``stuck'' agent from one that simply fails to attempt the task).

\noindent \textbf{Validation:} We conducted sanity checks on the classification accuracy using a manually labeled dataset of 40 examples (10 instances per failure type). As shown in Table~\ref{tab:validation_scores}, the analyzer demonstrates high agreement with human annotators, particularly in detecting procedural failures like ``Stuck in Midway'' and ``Wrong Termination'' ($100\%$ accuracy). Slight deviations occurred in the ``Fail to Attempt'' category (8/10), where the boundary between a very short unsuccessful attempt and zero attempt can be semantically ambiguous. Validation also included an analysis of action distributions for each failure type to confirm behavioral consistency across the classified subsets.

\subsection{Failure Number Analysis: The Mid-Stuck Challenge}

Quantitative analysis of the failure distribution for various agents (Gemini-2.5-pro, Gemma, and Gemma-SFT) demonstrates that \textbf{Get Stuck Midway} is the most significant challenge across all tested models.

\begin{figure}
    \centering
    \includegraphics[width=1.0\linewidth]{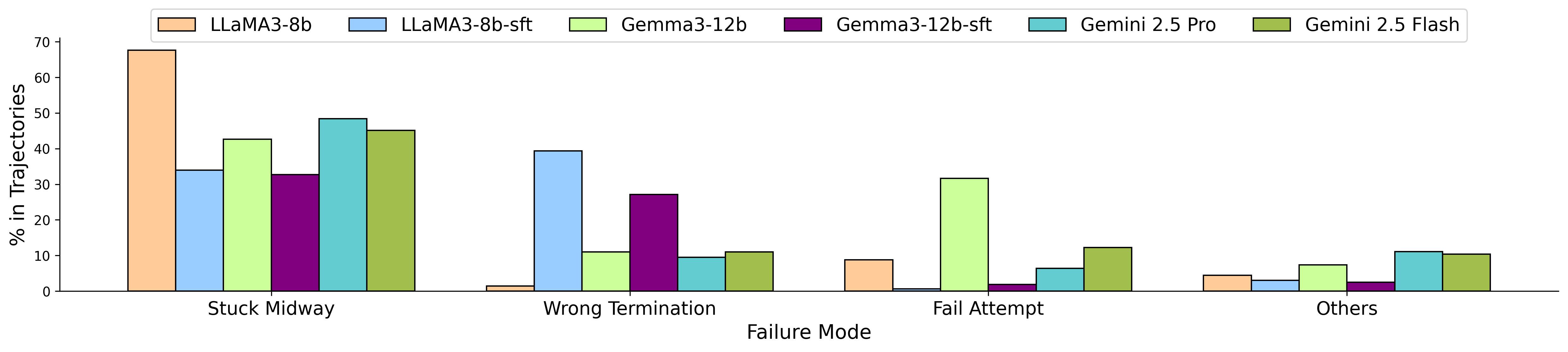}
    \captionsetup{justification=centering}
    \caption{Overview of Failure mode distribution of existing out-of-box models}
    \label{fig:dist}
\end{figure}

\begin{itemize}
    \item \textbf{Midway Dominance:} The majority of failures across all baseline models (Gemini, Gemma, and Gemma-SFT) stem from the \emph{Get Stuck Midway} mode, with failure ratios ranging from approximately $42\%$ to $49\%$. This mode is indicative of policy oscillation and poor long-term planning, particularly when the agent encounters unexpected states or is required to execute a long, complex sequence of actions.
    
    \item \textbf{Planning Deficiencies:} This high frequency suggests that existing policies, largely based on Supervised Fine-Tuning (SFT) or similar techniques, lack the robust planning and self-correction mechanisms necessary for sequential, multi-step web tasks. An agent stuck in a loop is fundamentally failing to update its internal state or subgoal, thereby repeating failed actions.
    
    \item \textbf{The Value of Subgoals and Planning:} The other primary failure mode, \emph{Fail to Make Reasonable Attempt}, while significant for Gemma (approximately $32\%$), is less dominant than the Midway failure. Agents that fail midway are often closer to the solution but lack the means to complete the final steps or break an incorrect loop. Therefore, policies integrated with explicit \emph{subgoal completion} or \emph{trajectory planning} signals are expected to yield the largest performance gains by directly targeting the pervasive \emph{Get Stuck Midway} category.
\end{itemize}

\section{Subgoal-Oriented Web Agent}

Building on the insights from our failure analysis, we hypothesize that \emph{subgoal setting serves as a critical mechanism for strengthening long-horizon reasoning and stabilizing web agent performance}. Specifically, we posit that introducing well-defined milestones can enhance both online interaction (through structured planning) and offline RL fine-tuning (through denser, intermediate rewards). To validate this hypothesis, we decompose our methodological contributions into three key components: generating reliable subgoals, improving inference-time performance using large proprietary models, and enhancing RL fine-tuning through subgoal-driven reward shaping. (Addressing \textbf{C.1})

\subsection{Subgoal Generation}
\label{sec:sg_gen}

\begin{wrapfigure}{r}{0.6\textwidth}
    \centering
    \vspace{-10pt}
    \includegraphics[width=\linewidth]{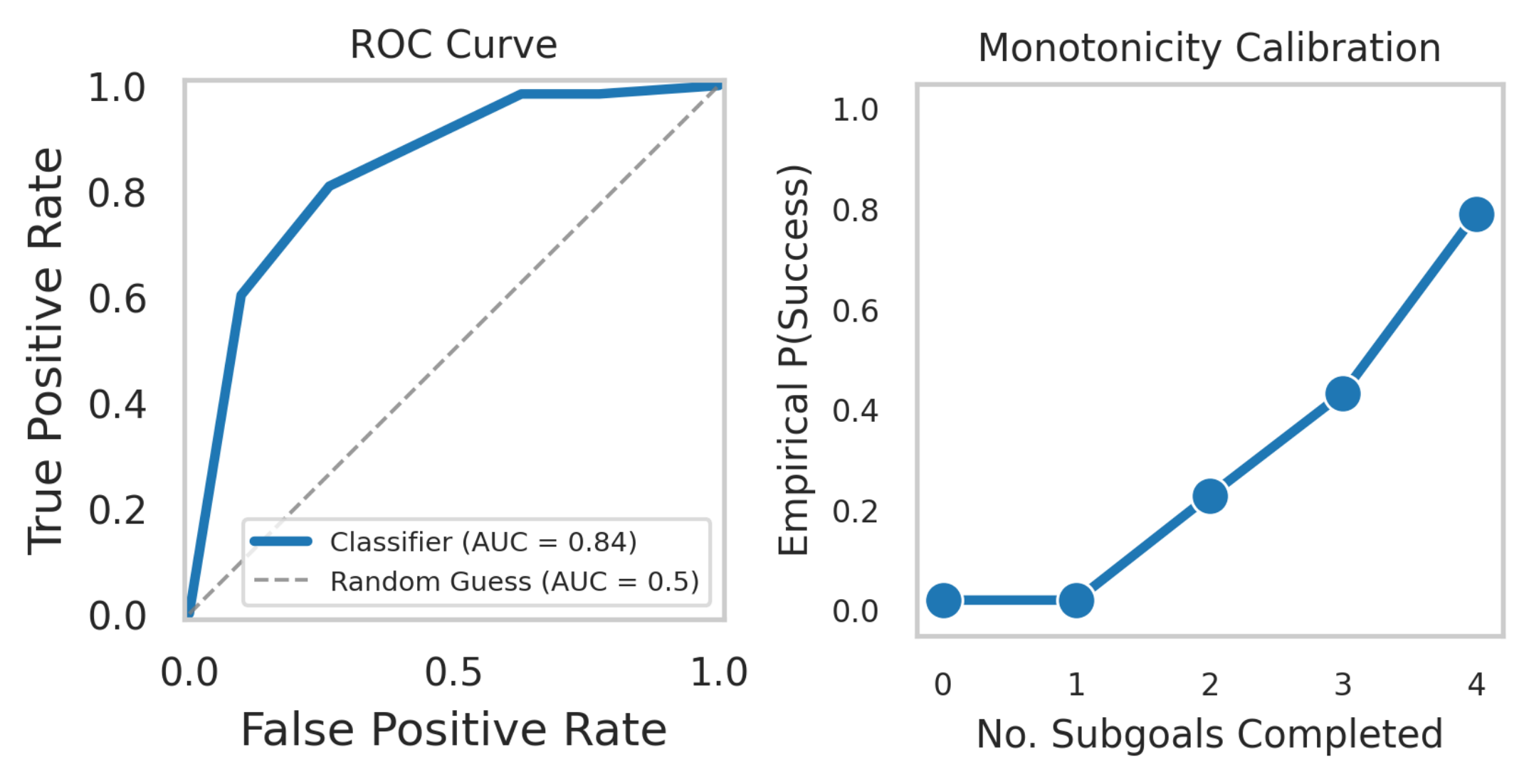}
    \caption{
        Validation of graded subgoal agreement. 
        \textbf{(Left)} The ROC curve, using the fraction of completed subgoals $s_i$ as a score, yields a high AUROC of 0.84. 
        \textbf{(Right)} The Monotonicity Calibration plot shows a strictly increasing probability of success $P(y=1 \mid m)$ as more subgoals $m$ are completed. 
        Both plots confirm the subgoals are a reliable progress signal.
    }
    \vspace{-10pt}
    \label{fig:graded_agreement_plots}
\end{wrapfigure}

To effectively guide web agents, we generate subgoals that bridge the gap between abstract user intents and precise step-by-step actions. Achieving this requires ensuring both reliability and proper granularity, which necessitates a fundamental understanding of website-specific constraints and features. Our focus isn't just on dividing the task, but on creating a clear sequence of progress signals that lead the agent successfully through long, complex interactions. We provide representative examples of these generated subgoals for each website category in WebArena in Table~\ref{tab:subgoal_examples} (Appendix~\ref{sec:apdx_subgoaleg}), and illustrate their online integration in Figure~\ref{fig:agent_architecture}.



\paragraph{Methodology.}
We leverage the teacher model, \textbf{Gemini-2.5-pro}, to generate subgoals given a high-level task description and the current Web state. To ensure robustness and generalization, we employ an iterative in-context learning strategy. Specifically, we curate a diverse set\footnote{2 examples for each website category} of few-shot demonstrations mapping intents to logical intermediate milestones. To mitigate positional bias and overfitting to specific action sequences, we inherently randomize the distribution of these few-shot examples during the prompting phase. This shuffling ensures the model attends to the \textit{semantics} of the task rather than memorizing rigid execution patterns. We then follow the validation process as shown below and iteratively optimize until we find the best generation contextual inputs for the teacher model.

\paragraph{Evaluation of Subgoal Quality.}
To validate the utility of our generated subgoals, we analyze the correlation between the generated subgoal vectors and the ground-truth final success across a validation set of agent traces\footnote{Agent is trained with MiRA and the base model is Gemma3-12b}. Let $i$ denote the index of a specific trace and $K$ be the total number of generated subgoals. We define the subgoal completion vector for trace $i$ as $\mathbf{z}_i \in \{0,1\}^K$, where $z_{i,k}$ indicates the completion status of the $k$-th subgoal. We assess the relationship between $\mathbf{z}_i$ and the final success label $y_i \in \{0,1\}$ on two levels: (1) \textit{Exact Equivalence}, testing if completing all subgoals is synonymous with success; and (2) \textit{Graded Agreement}, testing if partial completion serves as a reliable progress score.

\paragraph{Results and Analysis.}
Our analysis reveals that the generated subgoals function best as a continuous progress indicator rather than a strict binary gate.

First, under the \textit{Exact Equivalence} regime, we observe an Equivalence-F1 score of $0.6847$. While the Precision is high ($0.7917$), indicating that completing all subgoals is a strong indicator of success, the Recall is moderate ($0.6032$). This suggests that strict adherence to every generated subgoal is not a necessary condition for success; the agent occasionally finds valid alternative paths that bypass specific milestones (false negatives).

However, under the \textit{Graded Agreement} regime, the subgoals demonstrate exceptional utility. We define a progress score $s_i = \frac{1}{K}\sum_{k=1}^K z_{i,k}$ representing the fraction of subgoals completed for trace $i$.
\begin{itemize}
    \item \textbf{Discriminative Power:} The Area Under the ROC Curve (AUROC) for $s_i$ is \textbf{0.84} (Figure \ref{fig:graded_agreement_plots}, left). This high value indicates that the generated subgoals effectively rank-order traces, distinguishing promising trajectories from failures with high reliability.
    \item \textbf{Monotonicity:} As shown in Figure \ref{fig:graded_agreement_plots} (right), we observe a strictly monotonic relationship between the number of completed subgoals and the empirical probability of success. This is statistically confirmed by a Kendall's $\tau$ rank correlation of \textbf{0.4585} ($p < 0.001$).
\end{itemize}

\paragraph{Conclusion.}
These metrics confirm that our iterative prompting strategy with Gemini-2.5-pro yields subgoals that are logically consistent and predictive. While not acting as rigid constraints, they provide a dense, calibrated signal of progress ($s_i$), allowing the agent to estimate its proximity to success dynamically during execution.

Having established that our generated subgoals serve as a reliable, monotonic estimator of task progress, we now demonstrate how this signal is integrated during both online inference and offline training in the following two sections.

\subsection{Enhancing Online Inference with Dynamic Milestoning}
\label{subsec:inference_enhancement}

Standard web agents often suffer from error propagation in long-horizon tasks, where a single early mistake cascades into failure. To mitigate this, we introduce a real-time planning mechanism that forces the agent to ``think'' before acting, effectively turning the generated subgoals into dynamic checkpoints. This explicit state tracking is achieved through our Dynamic Milestoning Framework (Figure \ref{fig:agent_architecture}). In practice, we efficiently leverage the thinking capabilities of Gemini-2.5-pro, i.e. Gemini-SGO, to ground abstract plans into concrete execution traces. (Addressing \textbf{C.2})
\begin{figure}[t!]
    \centering
    \includegraphics[width=\linewidth]{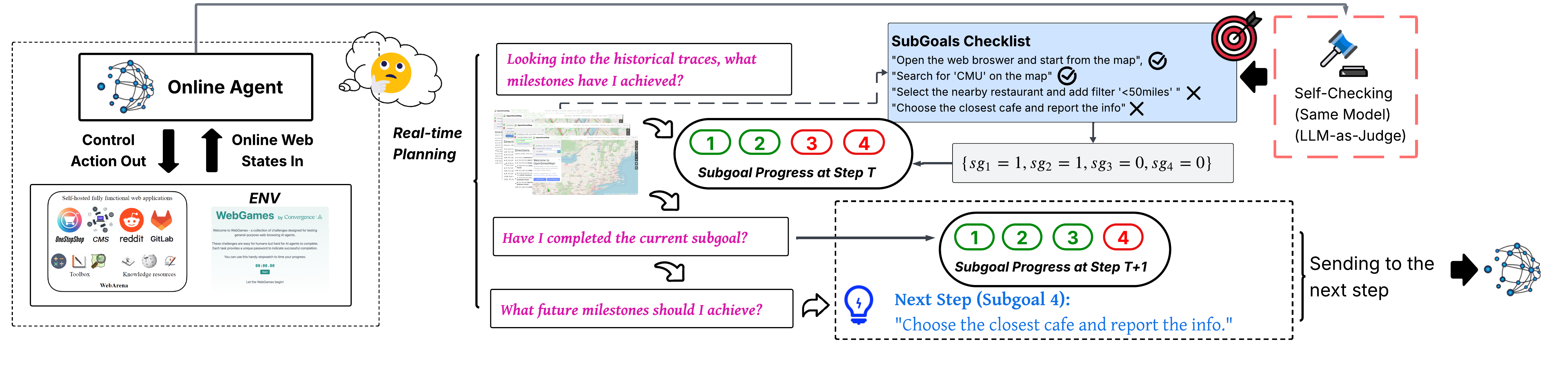}
    \caption{
        \textbf{Dynamic Milestoning Framework for Enhanced LLM Agent Inference.} 
        The architecture depicts the real-time feedback loop where the online agent's actions are monitored against a \textbf{SubGoals Checklist}. The reasoning model itself uses trace reflection to determine progress ($\mathbf{z}_{t+1}$), providing a dense, grounded signal that directs the agent's next planning step and enables self-correction.
    }
    \label{fig:agent_architecture}
\end{figure}

\paragraph{Grounding via Retrospective Reflection.}
Instead of blindly executing a static sequence, our agent engages in an iterative self-reflection process~\citep{shinn2023reflexion,madaan2023self,zhou2023language} at each time step $t$. As illustrated in our inference architecture, the agent queries its own historical interaction traces—comprising screenshots and action logs—to answer three critical questions:
\begin{enumerate}
    \item \textit{``Looking into the historical traces, what milestones have I achieved?''}
    \item \textit{``Have I completed the current subgoal?''}
    \item \textit{``What future milestones should I achieve?''}
\end{enumerate}

This process effectively grounds the high-level subgoals in the low-level reality of the web environment. By employing an \textit{AutoRater} (LLM-as-Judge) module instantiated within the same thinking model, the agent compares its current visual state against the \textit{SubGoals Checklist}. This allows the agent to maintain an explicit belief state of its progress for the current trace $i$, represented by the binary subgoal vector $\mathbf{z}_{i,t} = [z_{i,1}, z_{i,2}, \dots, z_{i,K}]$, where $z_{i,k} \in \{0,1\}$.

\paragraph{Contextual Grounding and Error Recovery.}

This milestoning mechanism enhances the online inference process by transforming the typically opaque execution state into a structured, explicit context for the thinking model. By injecting the verified subgoal status ($\mathbf{z}_{i,t}$) directly into the reasoning loop, we provide the agent with a ``situational awareness'' that formalizes the state reflection mechanisms explored in recent works~\cite{bishop2024latent,rawles2024androidworld}. If the \textit{AutoRater} determines that a required milestone $k$ is not yet satisfied ($z_{i,k} = 0$), this enhanced context prevents the agent from hallucinating progress. Instead, the model uses this state awareness to dynamically re-plan, generating actions specifically aimed at fulfilling the pending requirement. Conversely, confirmed milestones ($z_{i,k} = 1$) act as contextual anchors, allowing Gemini-2.5-pro to focus its planning capacity on the immediate next step rather than re-processing the entire task history. This ensures that every low-level control action is logically derived from a verified understanding of the current progress.

Additional analysis on the trade-offs between reasoning depth and inference efficiency is provided in Appendix~\ref{apdx:thinking_tradeoff}.

\subsection{Agent RL Finetuning with Subgoals (MiRA)}
\label{sec:mira_rl}


To incorporate subgoal-level feedback and provide dense reinforcement signals, MiRA extends the value estimation framework of prior work~\citep{qi2024webrl, bai2024digirl,wang2024distrl}. As illustrated in Figure~\ref{fig:mira_workflow}, the training pipeline introduces a dual-critic architecture: a standard \emph{Goal-conditioned value critic} $V_\phi(s,g)$ trained on binary outcomes (final success/failure), and a novel \emph{potential critic} $P_\psi(s,g)$ that acts as a dense \textbf{progress model} driven by a Sub-Goal Checker (Addressing \textbf{C.3}). 

If we visually track the workflow: trajectories generated in the \textbf{Interact} phase (left)—exemplified by the failures and successes—are fed into these critics. The critics' outputs are then synthesized via \textbf{Reward Shaping} to guide the \textbf{Actor Policy Update} (right), while training stability is ensured by applying \textbf{Actor Perplexity Filtering} to the \textbf{Experience Replay Buffer} (bottom), which would be used for both current and next phase of RL training (See $\S$\ref{sec:iterative_rl}). To better illustrate our novel design, we first explain from the transformation of discrete subgoal completion signals into progress labels in the following paragraphs.

\begin{figure}[t]
    \centering
    \includegraphics[width=\textwidth]{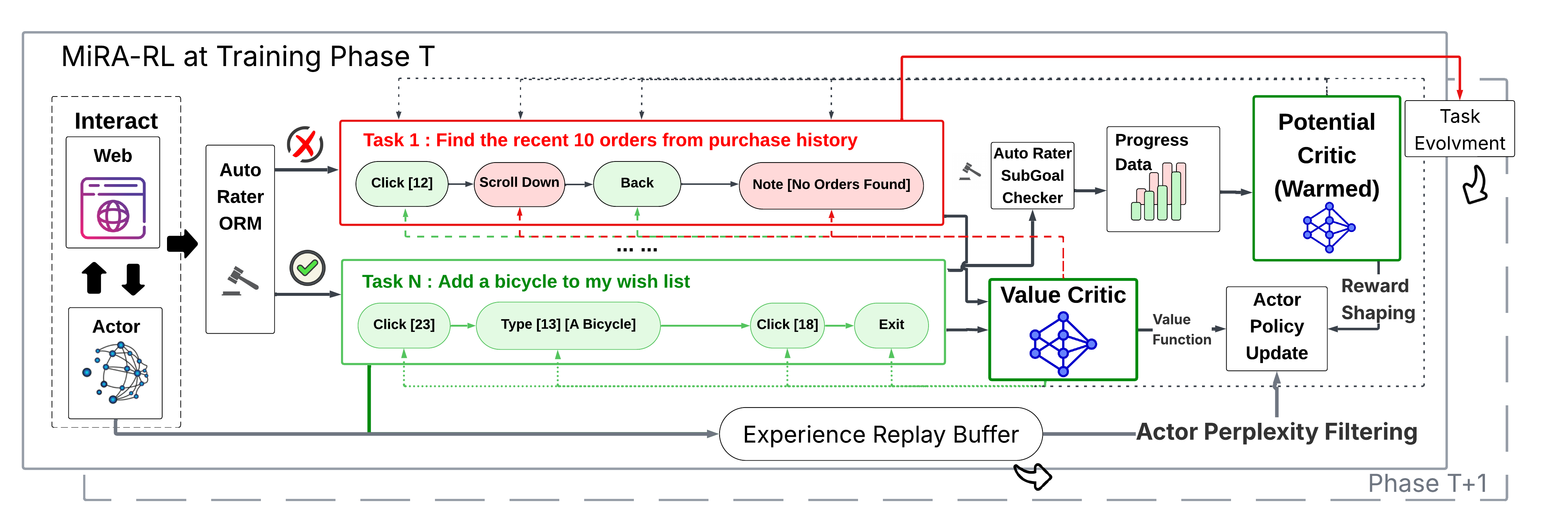}
 \caption{\textbf{The MiRA-RL Training Pipeline.} During the interaction phase, the agent generates trajectories—such as the successful ``Add a bicycle'' task or the failed ``Find recent orders'' task. These are evaluated by an Auto Rater (binary final success) and a SubGoal Checker (intermediate progress). This data trains two distinct critics: a \textbf{Value Critic} $V_\phi$ for final success and a \textbf{Potential Critic} $P_\psi$ that models progress. The Actor policy is updated using shaped rewards, with updates stabilized by Actor Perplexity Filtering.}
    \label{fig:mira_workflow}
\end{figure}


\paragraph{Progress Labeling via Subgoal Completion.}
For each trajectory and task goal $g$, we assume access to a binary subgoal--completion vector at each step, $\mathbf{z}_t = (z_{t,1}, \dots, z_{t,K}) \in \{0,1\}^K$, where $z_{t,k}=1$ indicates that the $k$-th subgoal has been achieved\footnote{Only positive traces are sampled}. Key steps are defined where the cumulative count $c_t$ increases ($c_t > c_{t-1}$).

\begin{figure}[t]
    \centering
    \includegraphics[width=1.0\textwidth]{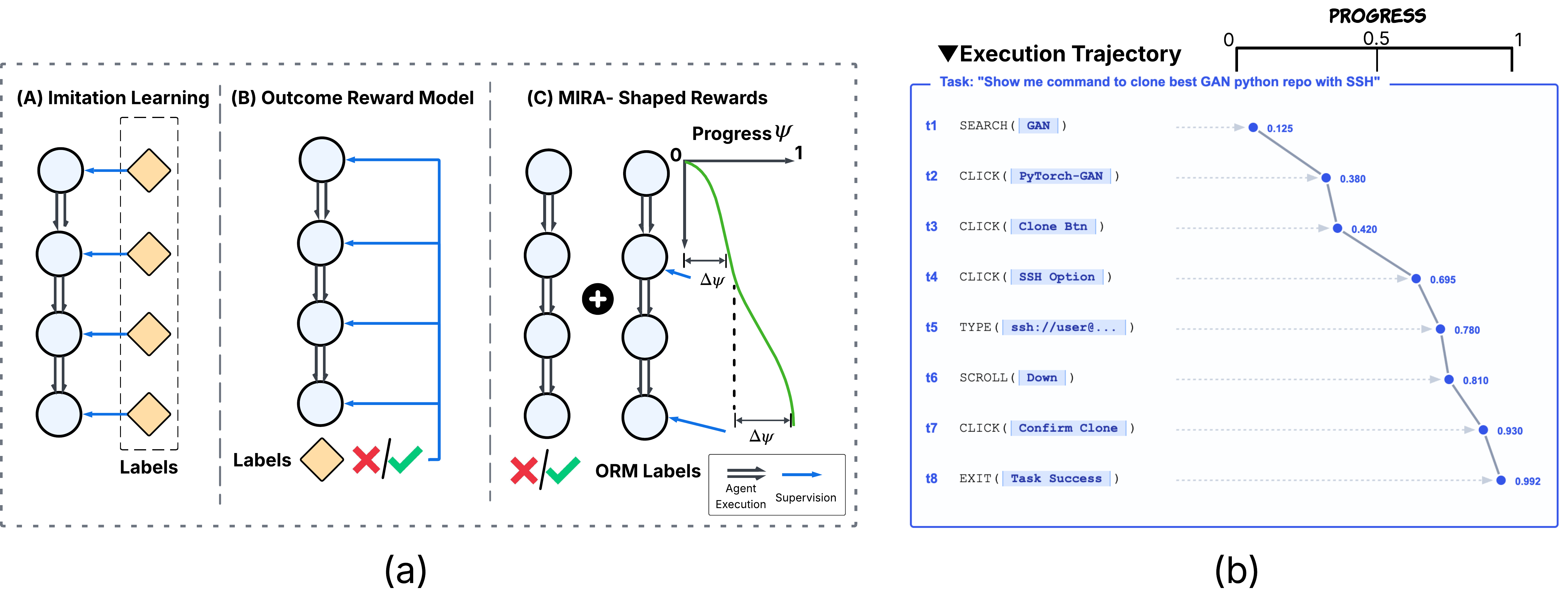}
    \caption{\textbf{Potential-Based Reward Shaping.} (Left) A comparison of supervision signals: (A) Imitation learning focuses on exact step matching; (B) Outcome Reward Models (ORM) provide sparse feedback only at termination; (C) MiRA utilizes Shaped Rewards via a learned potential function. (Right) An execution trajectory for a successful task ``Clone best GAN repo''. The predicted progress score $\psi$ (green curve) increases monotonically as the agent completes subgoals (e.g., ``Click Clone Btn'', ``Click SSH Option''). This generates a dense $\Delta\psi$ reward at every timestep $t$.}
    \label{fig:potential_concept}
\end{figure}

To obtain smooth supervision rather than a discontinuous step function, we assign normalized progress labels $p_t^\ast \in [0,1]$ through linear interpolation between key steps, as visualized by the green progress curve in Figure~\ref{fig:potential_concept} (Right). Let $t_j$ and $t_{j+1}$ be the steps where the $j$-th and $(j{+}1)$-th subgoals complete. For any $t\in[t_j,t_{j+1}]$:
\begin{equation}
    \alpha_t=\frac{t-t_j}{\,t_{j+1}-t_j\,},\quad
    p_t^\ast = (1-\alpha_t)\,\frac{j}{K} + \alpha_t\,\frac{j+1}{K}.
\end{equation}
This interpolation transforms discrete subgoal events (indicated by the diamond markers in Figure~\ref{fig:potential_concept}) into a continuous signal $\psi$, bridging the gap between sparse logic checks and dense RL updates. Figure~\ref{fig:potential_concept}(b) presents the predicted smooth potential scores from ``Gitlab'' tasks from WebArena.

\noindent \textbf{Example.} Consider a successful trajectory with $K=3$ subgoals. The first two subgoals complete at $t_1=2$ and $t_2=4$. Crucially, while the logic check for the final subgoal ($j=2$) passes at $t=6$, the agent performs necessary concluding actions (e.g., scrolling, verification, \texttt{exit}) until the true termination at $T=9$.
\begin{itemize}
    \item \textit{Segment 1 ($t \in [0, 2]$):} The agent progresses towards the first subgoal ($0 \to 0.33$). At $t=1$, $p_1^\ast \approx 0.16$.
    \item \textit{Segment 2 ($t \in [2, 4]$):} The agent progresses towards the second subgoal ($0.33 \to 0.66$). At $t=3$, $p_3^\ast = 0.5$.
    \item \textit{Segment 3 ($t \in [4, 9]$):} \textbf{Gap Anchoring.} To provide dense signal during the final administrative steps, we anchor the completion of the last subgoal to the trajectory end $T=9$ rather than $t=6$. The progress ramps from $0.66 \to 1.0$ over the interval $[4, 9]$. For a step inside the ``gap'' (e.g., $t=7$), $\alpha_7 = \frac{7-4}{9-4} = 0.6$, yielding a label $p_7^\ast = 0.4(0.66) + 0.6(1.0) \approx 0.86$.
\end{itemize}
These interpolated values $\{p_t^\ast\}_{t=1}^T$ serve as the \textbf{regression targets} for supervised training. By anchoring the final segment to $T$, we ensure the Potential Critic learns a monotonically increasing value function that drives the agent through the final verification steps to the sparse reward at termination.

\paragraph{Potential Critic for Dense Shaping.}
We train the \emph{potential critic} $P_\psi(s,g)\in[0,1]$ to regress onto these progress labels. As depicted in Figure~\ref{fig:potential_concept}.(a)-(C), this \emph{progress model} provides dense shaping signals distinct from the sparse ORM labels. While the standard ORM (Figure~\ref{fig:potential_concept}.(a)-(B)) effectively judges final validation, it fails to guide the agent during the intermediate steps. The potential critic fills this gap by estimating the ``distance'' to the next subgoal.

To construct a practical progress model, we augment a pretrained LLM with a multilayer perceptron (MLP) head and apply a sigmoid activation to constrain the output range to $[0,1]$. We employed Llama3-8b (WebRL~\citep{qi2024webrl}) and a vanilla-RL agent—operating without any milestone-based architectural modifications—to perform exploratory rollouts across 1,237 tasks. This collection phase yielded a diverse set of execution traces ranging from early failures to complete successes. Rather than relying on binary outcome labels, we post-processed these trajectories using a subgoal checker to derive dense supervision progress labels for every timestep. Finally, we utilized this granularly labeled dataset to fine-tune the Gemma-12B backbone in a single supervised learning stage, effectively distilling the subgoal logic into a differentiable potential critic.

\noindent \textbf{State Construction and Semantic Consistency.}
To enable reliable progress estimation, each state $s_t$ used by the potential critic must capture not only the most recent observation $o_t$ but also the \emph{semantic action history} leading to it. Concretely, we encode the state as
\(
s_t = [a_1, a_2, \ldots, a_{t-1},\, o_t],
\)
and append the natural-language goal instruction to form the model input $[h_{s_t}; e_g]$.  
A key observation from large-scale Web-agent rollouts is that \emph{successful trajectories solving the same task goal tend to exhibit a stable longest-common subsequence of actions}: even though surface-level behaviors vary across agents and tasks, the core semantic steps—such as ``open menu'', ``locate repository'', ``click clone'', ``verify output''—recur with strong regularity. These shared subsequences define the implicit structure of the task.  
By conditioning the potential critic on the full action–observation history, we allow it to learn this structure directly from data. As a result, states that lie along successful semantic paths are assigned higher potential values, which in turn supplies shaping rewards that encourage the agent to reproduce these key behaviors. In this way, the potential critic captures the latent procedural invariants of web tasks, providing dense guidance toward subgoal completion even when the environment supplies only sparse outcome-based signals.

We then utilize this pre-trained network as a \emph{warm-up} model. During the subsequent RL training, we continue to fine-tune the potential critic using fresh online rollouts generated by the current policy. This allows the potential landscape to evolve alongside the agent's capabilities. The critic is optimized via mean squared error:
\begin{equation}
\label{eq:potential_regression}
\mathcal{L}_P(\psi) = \mathbb{E}_{(s_t,g,p_t^\ast)} \Big[ \| P_\psi(s_t,g) - p_t^\ast \|^2 \Big].
\end{equation}



\noindent \textbf{Auxiliary Reward Shaping.}  
During RL, the potential critic provides an auxiliary shaping reward:
\[
r' = r + r_t^{(\mathrm{aux})} = r + \alpha \bigl[P_\psi(s_{t+1},g) - P_\psi(s_t,g)\bigr].
\]
This corresponds to the \(\Delta\psi\) arrows shown in Figure~\ref{fig:potential_concept}.(b). This term is added to the environment reward before computing advantage targets. Crucially, the main value critic \(V_\phi\) remains responsible for modeling final task success, ensuring that while the potential critic accelerates learning via dense feedback, it does not bias the optimal policy away from true task completion as recent research confirms that scaling potential-based auxiliary rewards is highly effective when grounded in success estimates. For example, SASR~\citep{ma2024highly} demonstrates that deriving shaped rewards directly from historical success rates significantly accelerates learning in sparse-reward settings without altering the optimal policy. Similarly, TIPS~\citep{anonymous2025tips} applies this design to LLM reasoning, using a ``teacher'' model to provide dense, scaled feedback based on the incremental probability of arriving at a correct solution.

\begin{designrationale}{Auxiliary Nature of Subgoals}  
Subgoal-based potentials offer dense feedback but do not modify the final optimization target. First, the main critic \(V_\phi\) is trained solely on the true task reward (binary success), while \(P_\psi\) influences only an additive shaping term. Second, \(P_\psi\) is trained entirely from post-hoc subgoal completions (the ``Auto Rater'' and ``SubGoal Checker'' path in Figure~\ref{fig:mira_workflow}) and carries no authority over final success. We further restrict shaping signals to those derived from \emph{positive traces}—i.e., trajectories that ultimately achieve the final goal—ensuring that subgoal completion is statistically correlated with task completion. If a subgoal does not correlate with final task success, its shaping influence naturally diminishes as the regression residual remains high.  
\end{designrationale}

\noindent
\textbf{Interpretation.}
The potential critic densifies supervision by predicting progress over long horizons, while the value critic captures the final success signal.  
Their combination yields advantages that remain aligned with the true task reward yet benefit from smoother intermediate credit assignment. The offline training workflow with subgoal-based progress shaping is presented Algorithm~\ref{alg:mira_curriculum}.


\paragraph{Goal-Conditioned Value Critic.}
The primary value critic $V_\phi(s,I)$ models the expected probability of successfully completing the overall task instruction $I$ from the current state $s$. Since the environment provides sparse, binary outcome labels $r(s_T, a_T, I) \in \{0,1\}$ (indicating failure or success), the value function naturally represents the probability of success, i.e., $V_\phi(s, I) \approx P(\text{success} \mid s, I) \in [0, 1]$.
To train this probabilistic estimator, we utilize a binary cross-entropy objective:
\begin{equation}
\label{eq:critic_value}
\mathcal{L}_V(\phi)
=
-\mathbb{E}_{(s,I)\sim\mathcal{D}}
\!\left[
r(s_T, a_T, I) \log V_\phi(s,I)
+
(1-r(s_T, a_T, I))\log(1-V_\phi(s,I))
\right].
\end{equation}
By quantifying the likelihood of success from the current observation $s$, this critic serves as a direct signal for the actor. It provides the baseline $V_\phi$ used in the computation of the advantage function. Following recent approaches~\citep{farebrother2024stop,bai2024digirl,qi2024webrl}, we adopt this classification-based formulation for training the value network, as opposed to standard regression, to better handle the sparse, binary nature of the terminal rewards.

\paragraph{Actor Optimization: Policy Update as a Supervised Regression.}
Our actor optimization strategy leverages the framework of relative entropy regularized reinforcement learning, treated as a variant of iterative Policy Mirror Descent (PMD)~\citep{mei2019principled,abbasi2019politex}. The rooted works~\citep{schulman2017equivalence, haarnoja2018soft} established a precise equivalence between ``soft'' Q-learning and policy gradient methods. Under this perspective, the policy update is modeled not as a simple gradient step, but as the solution to a constrained maximization problem: optimizing the expected return while minimizing the KL-divergence from a reference model $\pi_{\text{ref}}$.

This formulation is particularly powerful because it yields a closed-form analytical solution and was widely adopted in previous Agentic training paradigms~\citep{ouyang2022training, qi2024webrl, bai2024digirl, wang2024distrl, team2025kimi}. Following the derivation in these foundational efforts, the optimal policy $\pi^*$ for this objective is defined by the reference policy re-weighted by the exponential advantage function $A^*$. Formally, $\pi^*$ satisfies (simple proof can be found in Appendix~\ref{sec:apdx_proof}):
\begin{equation}
\label{eq:optimal_policy_exp_advantage}
\pi^*(a|s,I)=\pi_{\text{ref}}(a|s,I)\exp\!\left(\tfrac{1}{\beta}A^*(s,a,I)\right).
\end{equation}
To approximate $\pi^*$, we minimize the distance between the current policy $\pi_\theta$ and this optimal form, which yields the following mean-squared-error (MSE) regression objective:
\begin{align}
& \arg\min_{\theta}\mathbb{E}_{\nu}\!\left[\!\left(\log\pi_\theta(a|s,I)-\log\pi^*(a|s,I)\right)^{2}\!\right] \\
&=\arg\min_{\theta}\mathbb{E}_{\nu}\!\left[\!\left(\log\pi_\theta(a|s,I)-\log\pi_{\text{ref}}(a|s,I)-\tfrac{1}{\beta}A^*(s,a,I)\right)^{2}\!\right]\\
&=\arg\min_{\theta}\mathbb{E}_{\nu}\!\left[\!\left(\beta\log\tfrac{\pi_\theta(a|s,I)}{\pi_{\text{ref}}(a|s,I)}-A^*(s,a,I)\right)^{2}\!\right].
\label{eq:mse_objective}
\end{align}
This formulation is \textbf{off-policy}—the expectation $\mathbb{E}_{\nu}$ may be taken over any data distribution $\nu$, such as expert demonstrations or replay buffers.  
This regression-based view provides a stable and sample-efficient way to learn from large, static, and mixed-quality datasets where standard policy gradient methods often fail.  
It is closely related to preference-based objectives such as DPO~\citep{rafailov2023direct}, which also regress log-probability ratios toward externally derived preference signals.

\noindent \textbf{Interpretation.}
By minimizing Eq.~\ref{eq:mse_objective}, the model learns to align its scaled log-probability ratio with the advantage of an optimal policy:
\begin{equation}
\label{eq:log_ratio_equals_advantage}
\beta\,\log\tfrac{\pi_\theta(a_t|s_t,I)}{\pi_{\text{ref}}(a_t|s_t,I)}\approx A^*(s_t,a_t,I).
\end{equation}
This defines a supervised regression target, where the actor predicts the log-ratio and the critic (or return estimator) provides the advantage signal. Empirically, we observe that this regression-based formulation outperforms standard likelihood-ratio methods (e.g., PPO). We attribute this to the superior numerical stability of optimizing in log-space, which avoids the high-variance gradients associated with probability ratios, and the problem's off-policy nature, which allows for greater sample efficiency by effectively leveraging historical high-quality trajectories in the experience pool.

\noindent \textbf{Why not direct KL Divergence to measure the distance between $\pi_{\theta}$ and $\pi^*$?} \\
While KL divergence is a common measure of distance between policies in reinforcement learning, directly using it to optimize $\pi_\theta$ toward $\pi^*$ presents significant challenges. Specifically, optimizing
\(
\arg\max_\theta\mathbb{E}_{s\sim d(s), a\sim\pi^*(a|s,I)}[\log\pi_\theta]
\)
imposes strong restrictions on the data distribution. It would typically require the training data to conform to $\pi_{\text{ref}}$ for effective on-policy learning, or otherwise necessitate complex off-policy corrections.  
The MSE objective derived above (Eq.~\ref{eq:mse_objective}), by contrast, naturally supports off-policy data, aligning better with our goal of leveraging diverse datasets without stringent on-policy sampling requirements. 

Then, to implement this supervised objective in practice, we first require an empirical estimate for the target advantage $A^*(s_t,a_t,I)$, followed by a gradient-based update to the actor network.

\begin{algorithm}[t!]
\caption{\textsc{MiRA-RL}: Actor Optimization with Potential-Based Reward Shaping (Inner Loop)}
\label{alg:mira-rl}
\begin{algorithmic}[1]
\Require Training batch $\mathcal{D}_{\text{train}}$; value critic $V_\phi$; potential critic $P_\psi$; current policy $\pi_\theta$; reference policy $\pi_{\mathrm{ref}}$; discount $\gamma$; temperature $\beta$; mixing $\lambda$; learning rates $\eta$.
\vspace{0.25em}

\State \textbf{Initialize} local optimization step counter $t \gets 0$
\While{optimization not converged}
  \State \textbf{Phase A: Critic updates}
  \For{$i=1$ to $N_c$}
    \State Sample minibatch $\mathcal{B}_c \subset \mathcal{D}_{\text{train}}$
    \State Update Value Critic $\mathcal{L}_V(\phi)$ via Eq.~\eqref{eq:critic_value}
    \State Update Potential Critic $\mathcal{L}_P(\psi)$ via Eq.~\eqref{eq:potential_regression}
    \State $\phi \gets \phi - \eta_V \nabla \mathcal{L}_V$, \quad $\psi \gets \psi - \eta_P \nabla \mathcal{L}_P$
  \EndFor

  \State \textbf{Phase B: Reward Shaping \& Returns}
  \State Sample minibatch $\mathcal{B}_a \subset \mathcal{D}_{\text{train}}$
  \For{each trajectory $\tau \in \mathcal{B}_a$}
    \State Calculate potential diffs: $r'_t \gets r_t + \alpha(P_\psi(s_{t+1}, g) - P_\psi(s_t, g))$
    \State Compute returns $G_t$ using shaped rewards $r'_t$
  \EndFor

  \State \textbf{Phase C: Actor Update}
  \For{each transition in $\mathcal{B}_a$}
    \State Compute Advantage $A_t^{\text{shaped}}$ (Eq.~\ref{eq:V4A}) mixing Monte-Carlo and Critic
    \State Compute log-ratio $\rho_t$ against $\pi_{\mathrm{ref}}$
  \EndFor
  \State Update $\theta \gets \theta - \eta_\pi \nabla_\theta \mathcal{L}_\pi(\theta)$ (minimizing Eq.~\ref{eq:mse_objective})
\EndWhile
\State \Return Updated parameters $(\pi_\theta, V_\phi, P_\psi)$

\end{algorithmic}
\end{algorithm}

\paragraph{Advantage Target Estimation.}
We compute the advantage target, $A_t^{\text{shaped}}$, using a doubly-robust estimator that mixes a 1-step TD-error term with a full Monte-Carlo return. This balances the low variance of TD learning with the high-variance, unbiased signal of long-horizon returns:
\begin{equation}
\label{eq:V4A}
A_t^{\text{shaped}}
=
\lambda
\underbrace{\big(
r'_t + \gamma V_\phi(s',g) - V_\phi(s,g)
\big)}_{\text{1-step TD Error}}
+
(1-\lambda)
\underbrace{\big(
G_t - V_\phi(s,g)
\big)}_{\text{MC Advantage}},
\end{equation}
with the full Monte-Carlo return $G_t$ defined as:
\begin{equation}
G_t = \sum_{u=t}^{T} \gamma^{u-t} r'_u .
\end{equation}
This $A_t^{\text{shaped}}$ value provides the dense, shaped signal that our policy will be regressed against.

\paragraph{Policy Update Implementation} 

Given the regression objective in Eq.~\ref{eq:mse_objective}, the corresponding policy loss can be written as:
\begin{equation}
\label{eq:actor_loss_final}
\mathcal{L}_\pi(\theta)
=
\mathbb{E}_{\nu}\!
\left[
\left(
\beta\,\log\tfrac{\pi_\theta(a\mid s,I)}{\pi_{\mathrm{ref}}(a\mid s,I)}
-
A_t^{\text{shaped}}
\right)^{2}
\right],
\end{equation}
where $\nu(s)$ denotes the distribution of experiences (e.g., replay or mixed expert data).  
Our method operates fully in the off-policy setting without requiring data from the current policy.
To gain a mechanistic understanding of the regression objective, we inspect its gradient:
\begin{align}
\nabla_\theta \mathcal{L}_\pi
&=
-\,2\beta\,
\mathbb{E}_{(s,a)\sim\nu}\!\Bigg[
\Big(
\underbrace{A_t^{\text{shaped}}}_{\text{update direction}}
-
\underbrace{\beta \,\log \tfrac{\pi_\theta(a\mid s,I)}{\pi_{\mathrm{ref}}(a\mid s,I)}}_{\text{KL divergence constraint}}
\Big)
\;
\nabla_\theta \log \pi_\theta(a\mid s,I)
\Bigg].
\label{eq:actor_gradient_annotated}
\end{align}

\noindent
\textbf{Interpretation.}
This gradient update reveals three important mechanisms:

\begin{itemize}
    \item \textbf{Advantage-Guided Update.}
    When $A(s,a,I)>0$, the action $a$ is valuable, so its log-probability should increase.  
    The magnitude of the update scales with the advantage gap: the larger $A^*$, the more strongly $\pi_\theta$ shifts toward increasing the action likelihood.  
    Conversely, when $A(s,a,I)<0$, the update pushes the policy to reduce its probability for suboptimal actions.

    \item \textbf{KL-Constrained Regularization.}
    The term $\log\tfrac{\pi_\theta}{\pi_{\mathrm{ref}}}$ imposes a soft regularization that prevents excessive policy deviation.  
    When $\pi_\theta$ already assigns higher probability than $\pi_{\mathrm{ref}}$, the corrective signal $(A^*-\beta\log\tfrac{\pi_\theta}{\pi_{\mathrm{ref}}})$ is reduced, stabilizing updates and avoiding overshooting—analogous to a trust-region constraint.

\end{itemize}

Overall, this formulation can be viewed as a KL-regularized regression policy update, where the actor increases the log-probability of advantageous actions relative to the reference while suppressing those that yield negative advantages, achieving stable off-policy learning. This optimization procedure constitutes the \emph{inner loop} of our framework; in the following section, we detail the \emph{outer curriculum loop}, which iteratively refreshes the training distribution based on failure analysis to progressively patch the agent's weakness.

\subsection{Iterative Policy Refinement Cycle}
\label{sec:iterative_rl}

A key limitation of static, off-policy reinforcement learning is its vulnerability to local optima: once an agent saturates performance on a fixed task set, it ceases to acquire new skills and fails to generalize to novel task configurations~\citep{qi2024webrl}. To address this bottleneck, we propose an \textbf{iterative refinement framework} in which the agent's policy is periodically patched through a curriculum of failure-driven updates. This design transforms the training process into an \emph{online curriculum of offline RL phases}, where each iteration learns from the agent's most recent failures to continually expand its competence boundary.

\begin{figure}[t!]
    \centering
    \includegraphics[width=1.0\linewidth]{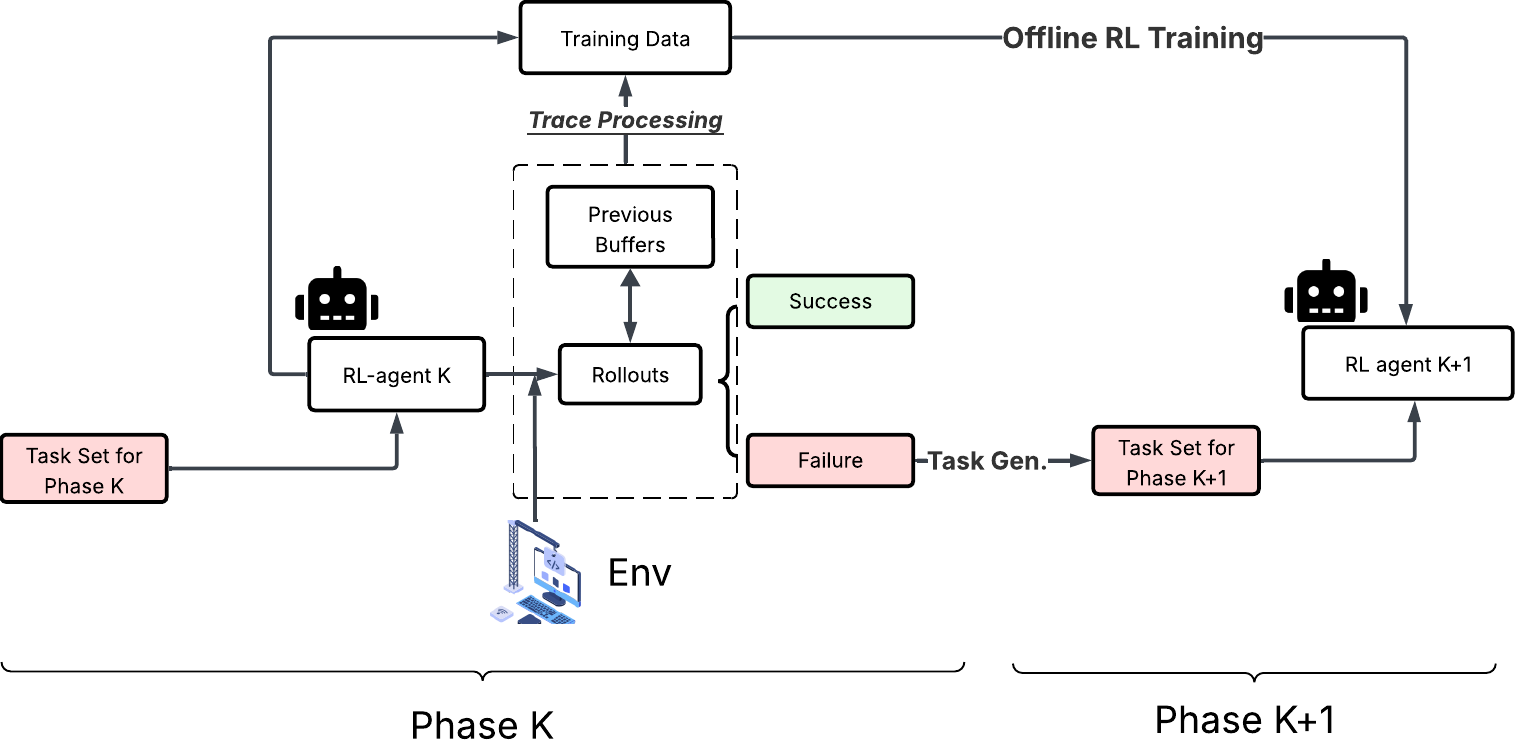}
    \caption{Online Curriculum Training Pipeline for Patching the Model. 
    The framework alternates between \emph{offline RL training} and \emph{online environment interaction}. Failed trajectories are analyzed to generate harder task distributions for subsequent phases.}
    \label{fig:curriculum}
\end{figure}

As illustrated in Figure~\ref{fig:curriculum}, the workflow proceeds cyclically from an existing policy $\text{RL-agent } K$ to a refined successor $\text{RL-agent } K\!+\!1$ via coordinated phases of environment interaction, offline policy updates, and curriculum adaptation. In each phase, the current policy interacts with the environment to collect trajectories, which are processed and aggregated into a training corpus. The agent is then refined through offline batched RL updates before being redeployed on a newly generated task distribution derived from previous failures. Over successive phases, this failure-driven curriculum gradually expands the task distribution while improving the agent's robustness and long-horizon reasoning ability. Additional implementation details—including rollout processing, perplexity filtering, task generation, and the full training loop—are provided in Appendix~\ref{apdx:iterative_rl}.

\section{Experiments and Results}

\subsection{Experimental Setup}

\paragraph{Benchmarks and Environments.} 
We evaluate the effectiveness of our proposed framework using WebArena-Lite \citep{liu2024visualagentbench}, a rigorously curated subset of the WebArena benchmark \citep{zhou2023webarena}. The evaluation suite comprises 165 tasks distributed across five distinct, real-world application domains: \textit{Shopping Admin} (35), \textit{Map} (26), \textit{Shopping} (45), \textit{Reddit} (19), and \textit{Gitlab} (30). We prioritize this refined subset over the original 812-task WebArena suite to avoid issues with task feasibility and unstable evaluation. Prior analyses show that many tasks in the full benchmark have underspecified goals or rely on unsupported backend functionality, introducing noise that obscures true agent performance. Using WebArena-Lite ensures that failures reflect agent reasoning rather than environment faults. The reduced task set also lowers evaluation time from over six hours to about 40 minutes, enabling the frequent validation needed for our iterative curriculum design.

\paragraph{Baselines and Model Architectures.} 
To establish a comprehensive performance landscape, we compare our methods across two axes: (i) inference-level enhancement via our subgoal-oriented planning framework (SGO) and (ii) training-level curriculum enhancement via MiRA.  
In the proprietary category, we evaluate GPT-4-Turbo, GPT-4o, and the Gemini-2.5 family (Flash and Pro). For open-source agents, we employ Llama-3.1 (8B) and Gemma-3 (12B) as foundational backbones. We contrast our method against standard Supervised Fine-Tuning (SFT) baselines, including advanced reinforcement learning approaches such as AWR \citep{peng2019advantage}, DigiRL \citep{bai2024digirl}, and the current state-of-the-art WebRL \citep{qi2024webrl}.

\paragraph{Training Protocol.}
To ensure a fair and rigorous comparison, we strictly control the initialization conditions for all learning-based methods. Both the baseline RL agents (WebRL, DigiRL) and our MiRA variants are initialized from the same SFT checkpoints and same SFT data set (Phase 0)—specifically, Llama-3.1-SFT and Gemma-3-SFT—before undergoing their respective reinforcement learning phases. The critic networks share the actor's backbone with the addition of randomly initialized value heads. For MiRA, the potential critic is pre-trained on the offline dataset to establish the initial shaping landscape before the online curriculum commences. Moreover, we reproduce the DigiRL framework within the WebArena-Lite environment. We retain the identical architectural components described in the original work, including the Advantage-Weighted Regression (AWR) optimization, both instruction-level and step-level value functions, and the experience replay buffer. To facilitate direct comparison, we adjust the data format to align with the Web-Nav standard and execute the training 6 rounds of online interaction. For WebRL, we adhere to its native phase-aware training curriculum. Note that for the Llama-3.1 WebRL baseline, we utilize the official open-source checkpoint available on Hugging Face; conversely, for Gemma-3, we train the agent from scratch following the same curriculum due to the absence of a pre-existing checkpoint. In contrast, the remaining baselines (such as SFT and standard AWR) undergo a ``one-go'' training protocol without iterative online refinement. Consequently, in our success rate (SR) evolution curves, these static baselines are visualized as dashed horizontal lines to serve as fixed performance references.

\subsection{Main Results}

\subsubsection{End-to-End Subgoal-Oriented Web Navigation Efficiency}

We present the main comparison results in Table~\ref{tab:model-comparison}. Our evaluation demonstrates that both the inference-level optimization (SGO) and the training-level curriculum (MiRA) yield substantial gains.

\begin{table}[t]
    \centering
    \footnotesize
    \caption{\textbf{Task Success Rate (SR) Comparison (WebArena-Lite).} Comparison of our models (Gemini-SGO and Gemma3 + MiRA) against proprietary and open-sourced baselines. Our models demonstrate superior performance in their respective categories. We report Pass@1 here to emphasize immediate task completion capability.}
    \label{tab:model-comparison}
    \renewcommand{\arraystretch}{1.2} 
    \setlength{\tabcolsep}{5pt}       

    \begin{tabular}{l c c c c c c c}
        \toprule
        \textbf{Models} & \textbf{\#Params} & \textbf{Reddit} & \textbf{Gitlab} & \textbf{CMS} & \textbf{Map} & \textbf{OSS} & \textbf{Avg. SR} \\
        \midrule

        \multicolumn{8}{c}{\textit{Proprietary LLMs}} \\
        \midrule
        GPT-4-Turbo & N/A & 10.5 & 16.7 & 14.3 & 36.7 & 13.3 & 17.6 \\
        GPT-4o & N/A & 10.5 & 10.0 & 20.0 & 20.0 & 11.1 & 13.9 \\
        Gemini-2.5-flash & N/A & 21.1 & 33.3 & 14.3 & 11.5 & 24.4 & 20.6 \\
        Gemini-2.5-pro & N/A & 31.6 & 30.0 & 34.3 & 15.4 & 13.3 & 23.0 \\
        \rowcolor{blue!5}
        \textbf{Gemini-2.5-pro-SGO (ours)} & \textbf{N/A} & \textbf{\underline{26.3}} & \textbf{\underline{50.0}} & \textbf{\underline{37.1}} & \textbf{\underline{23.1}} & \textbf{\underline{28.9}} & \textbf{\underline{32.1}} \\

        \midrule

        \multicolumn{8}{c}{\textit{Small Open-sourced LLMs}} \\
        \midrule
        AutoWebGLM (Lai et al., 2024) & 6B & 9.4 & 15.0 & 28.6 & 24.8 & 17.1 & 18.2 \\
        GLM-4-Chat (GLM et al., 2024) & 9B & 5.3 & 10.0 & 6.7 & 3.3 & 6.7 & 6.1 \\
        GLM-4 + SFT (BC) & 9B & 47.4 & 13.3 & 31.4 & 23.3 & 13.3 & 22.4 \\
        Llama3.1 + SFT (BC) & 8B & 36.8 & 6.7 & 20.0 & 33.3 & 17.8 & 20.6 \\
        Llama3.1 + AWR (Peng et al., 2019) & 8B & 57.9 & 26.7 & 31.4 & 26.7 & 17.8 & 28.5 \\
        Llama3.1 + DigiRL (Bai et al., 2024) & 8B & 57.9 & 26.7 & 37.1 & 33.3 & 17.8 & 30.3 \\
        Llama3.1 + WebRL (Qi et al., 2025) & 8B & 68.4 & 36.7 & 51.4 & 15.4 & \textbf{\underline{40.0}} & 38.8 \\
        \midrule
        Gemma3 + SFT (BC) & 12B & 52.6 & 40.0 & 42.9 & 23.1 & 17.8 & 30.9 \\
        Gemma3 + DigiRL & 12B & 52.6 & 46.7 & 37.1 & \textbf{\underline{34.6}} & 20.0 & 33.3 \\
        Gemma3 + WebRL & 12B & 68.4 & 43.3 & 40.0 & 30.8 & 20.0 & 35.1 \\
        \rowcolor{red!5}
        \textbf{Gemma3 + MiRA (ours)} & \textbf{12B} & \textbf{\underline{73.7}} & \textbf{\underline{56.7}} & \textbf{\underline{54.3}} & 30.8 & 28.9 & \textbf{\underline{43.0}} \\
        \bottomrule
    \end{tabular}
\end{table}

Comparison with open-source models reveals that Gemma3 + MiRA (12B) achieves the highest average success rate of \textbf{43.0\%}, substantially outperforming the WebRL baseline (35.1\%) and DigiRL (33.3\%). This margin validates the efficacy of our potential-based reward shaping, particularly in complex domains such as \textit{Gitlab} (56.7\%) and \textit{Shopping Admin} (54.3\%), where the model must adhere to strict procedural dependencies that purely sparse-reward methods often fail to capture. Similarly, in the proprietary regime, our Gemini-SGO method improves upon the base Gemini-2.5-pro model by nearly 10 percentage points (23.0\% → 32.1\%), confirming that our subgoal-oriented optimization generalizes across model architectures.

\paragraph{Complementary Benefits of Offline Refinement and Online Planning.}
The strong performance of both Gemma3-MiRA and Gemini-SGO highlights their distinct contributions across the agent's lifecycle. The offline RL phase (MiRA) allows the model to internalize subgoal dependencies into its weights, effectively ``compiling'' planning into intuition for common web navigations. The inference-time mechanism (SGO), conversely, serves as a runtime guardrail that is independent of the training phase. As detailed in Section~\ref{subsec:inference_enhancement} (and illustrated in Fig.~\ref{fig:gemini-thinking-trace}), the Dynamic Milestoning framework enables the agent to verify state transitions before proceeding. This real-time milestoning offers several key advantages: it (i) avoids blind exploration by checking whether a subgoal is achieved, (ii) enables corrective feedback mid-trajectory if progress stalls or deviates, and (iii) reduces planning uncertainty in long-horizon tasks by breaking workflows into verifiable micro-milestones. The combined effect is crucial for long-horizon web interactions: the trained policy provides strong priors to navigate efficiently, while the independent inference-time LLM catches and corrects deviations in real-time, effectively preventing the error propagation typical of purely reactive agents.

Complementing the efficiency analysis, Table~\ref{tab:gemini_failures} highlights the qualitative shift in agent behavior. Notably, the SGO framework drives a significant reduction in \textit{Stuck Midway} errors, dropping to \textbf{39.87\%} compared to the base Gemini 2.5 Pro (48.41\%) and Flash (45.12\%). This reduction directly validates the utility of the Dynamic Milestoning mechanism in helping the agent break out of local optima and navigational dead-ends. While this increased throughput results in a marginal rise in \textit{Wrong Terminations} (9.52\% $\rightarrow$ 12.03\%)—as the agent is now capable of reaching terminal states that were previously inaccessible—it maintains a low rate of \textit{Fail Attempts} (6.96\%) comparable to the Pro baseline, confirming that the dynamic compute allocation improves trajectory completion without sacrificing basic instruction adherence.

\begin{table}[h]
\centering
\small
\caption{Failure distribution ratio in all traces (\%) across Gemini model variants. The SGO framework significantly reduces navigation stalls compared to the base models.}
\label{tab:gemini_failures}
\begin{tabular}{lccc}
\toprule
\textbf{Failure Mode} & \textbf{Gemini 2.5 Flash} & \textbf{Gemini 2.5 Pro} & \textbf{Gemini 2.5 SGO} \\
\midrule
Stuck Midway & 45.12 & 48.41 & \textbf{39.87} \\
Wrong Termination & 10.98 & 9.52 & 12.03 \\
Fail Attempt & 12.20 & 6.35 & 6.96 \\
Others & 10.37 & 11.11 & 8.86 \\
\bottomrule
\end{tabular}
\end{table}

\paragraph{Inference Efficiency and Compute Trade-offs.}
Deploying such explicit reasoning at inference time, however, introduces a latency trade-off. Our analysis of ``Thinking Budgets'' (Figure~\ref{fig:thinking_tradeoff}) reveals that while naively increasing the static computation budget per step improves success rates—peaking at $\approx 32.5\%$ with 8192 tokens—it drastically increases inference latency to over 19 seconds per step, yielding diminishing returns. The \emph{Auto (Dynamic)} strategy employed by our Gemini-SGO agent effectively resolves this optimization problem. By adaptively allocating compute only when milestone verification is ambiguous, the dynamic model achieves success rates statistically comparable to the maximum static budget ($\approx 32.1\%$) but with significantly lower latency. This confirms that our framework succeeds not merely by scaling inference compute, but by intelligently shifting the burden of planning between the amortized cost of offline training and the targeted application of online reasoning. In the following sections, we will delve deeper into the MiRA training phases, domain-specific effects, ablations on subgoal modules and subgoal completion dynamics.

\begin{figure*}[t]
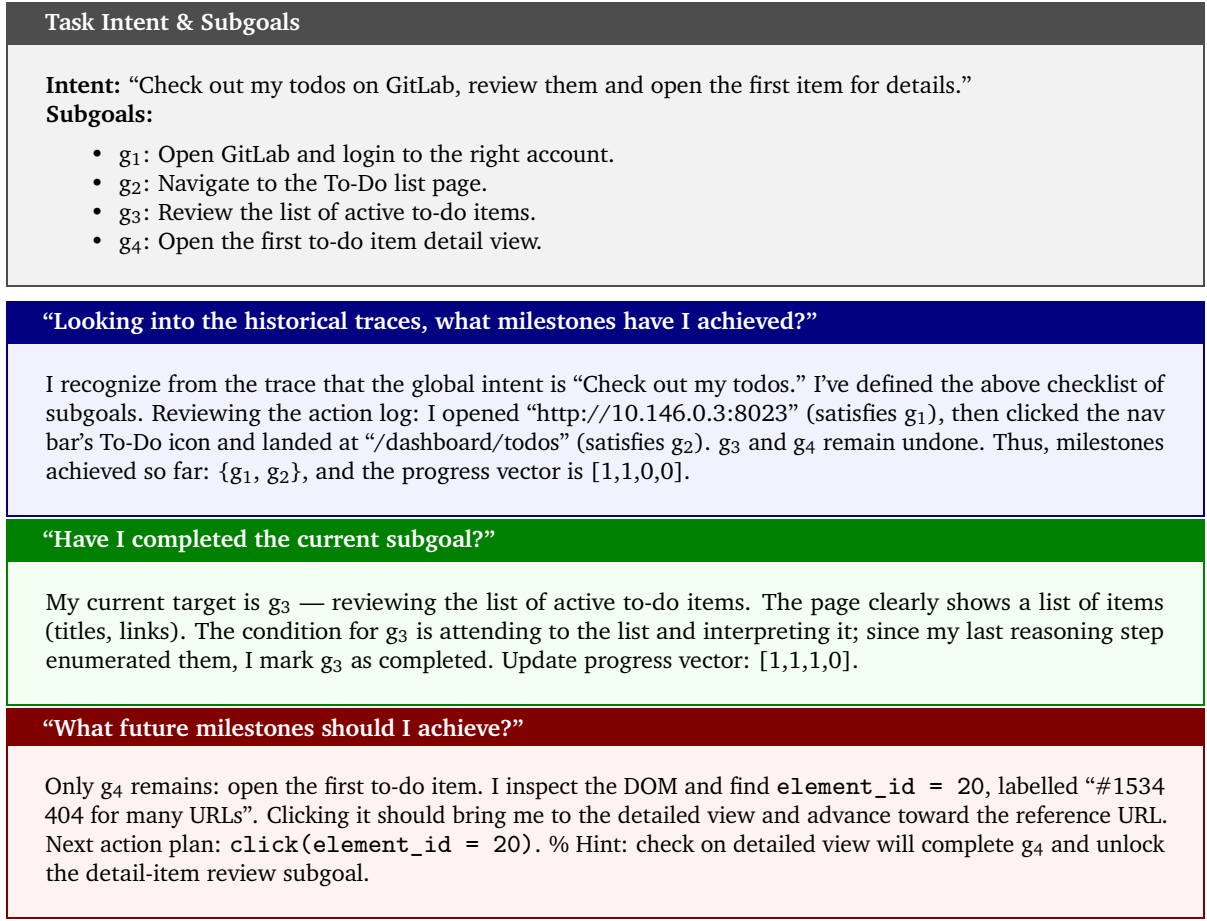

\centering
\begin{minipage}{0.95\textwidth}
\footnotesize

\begin{tcolorbox}[
    colback=gray!10,
    colframe=gray!60!black,
    sharp corners,
    boxrule=0.8pt,
    enhanced,
    left=4mm,right=4mm,top=3mm,bottom=3mm,
    title=\textbf{Task Intent \& Subgoals},
    fonttitle=\bfseries
]
\textbf{Intent:} ``Check out my todos on GitLab, review them and open the first item for details.'' \\
\textbf{Subgoals:}
\begin{itemize}
  \item g\textsubscript{1}: Open GitLab and login to the right account.
  \item g\textsubscript{2}: Navigate to the To-Do list page.
  \item g\textsubscript{3}: Review the list of active to-do items.
  \item g\textsubscript{4}: Open the first to-do item detail view.
\end{itemize}
\end{tcolorbox}

\begin{tcolorbox}[
    colback=blue!5,
    colframe=blue!50!black,
    sharp corners,
    boxrule=0.8pt,
    enhanced,
    left=4mm,right=4mm,top=3mm,bottom=3mm,
    title=\textbf{``Looking into the historical traces, what milestones have I achieved?''},
    fonttitle=\bfseries
]
I recognize from the trace that the global intent is ``Check out my todos.'' I've defined the above checklist of subgoals. Reviewing the action log: I opened ``http://10.146.0.3:8023'' (satisfies g\textsubscript{1}), then clicked the nav bar's To-Do icon and landed at ``/dashboard/todos'' (satisfies g\textsubscript{2}). g\textsubscript{3} and g\textsubscript{4} remain undone. Thus, milestones achieved so far: \{g\textsubscript{1}, g\textsubscript{2}\}, and the progress vector is [1,1,0,0].
\end{tcolorbox}
\vspace{-10pt}

\begin{tcolorbox}[
    colback=green!5,
    colframe=green!50!black,
    sharp corners,
    boxrule=0.8pt,
    enhanced,
    left=4mm,right=4mm,top=3mm,bottom=3mm,
    title=\textbf{``Have I completed the current subgoal?''},
    fonttitle=\bfseries
]
My current target is g\textsubscript{3} — reviewing the list of active to-do items. The page clearly shows a list of items (titles, links). The condition for g\textsubscript{3} is attending to the list and interpreting it; since my last reasoning step enumerated them, I mark g\textsubscript{3} as completed. Update progress vector: [1,1,1,0].
\end{tcolorbox}
\vspace{-10pt}

\begin{tcolorbox}[
    colback=red!5,
    colframe=red!50!black,
    sharp corners,
    boxrule=0.8pt,
    enhanced,
    left=4mm,right=4mm,top=3mm,bottom=3mm,
    title=\textbf{``What future milestones should I achieve?''},
    fonttitle=\bfseries
]
Only g\textsubscript{4} remains: open the first to-do item. I inspect the DOM and find \texttt{element\_id = 20}, labelled ``\#1534 404 for many URLs''. Clicking it should bring me to the detailed view and advance toward the reference URL.  
Next action plan: \texttt{click(element\_id = 20)}. \% Hint: check on detailed view will complete g\textsubscript{4} and unlock the detail-item review subgoal.
\end{tcolorbox}

\end{minipage}
\captionsetup{justification=centering}
\caption{\textbf{Showcase of Introspective planning of the Gemini-SGO-based Web Agent.}}
\label{fig:gemini-thinking-trace}
\end{figure*}

\subsubsection{Multi-Phase MiRA Performance and Domain-Level Trends}

To better understand how MiRA improves through iterative curriculum refinement, we visualize both the overall success rate (SR) across training phases and the per-domain trends on WebArena-Lite.

\begin{figure}[t!]
    \centering
    \begin{subfigure}[t]{0.47\linewidth}
        \centering
        \includegraphics[width=\linewidth]{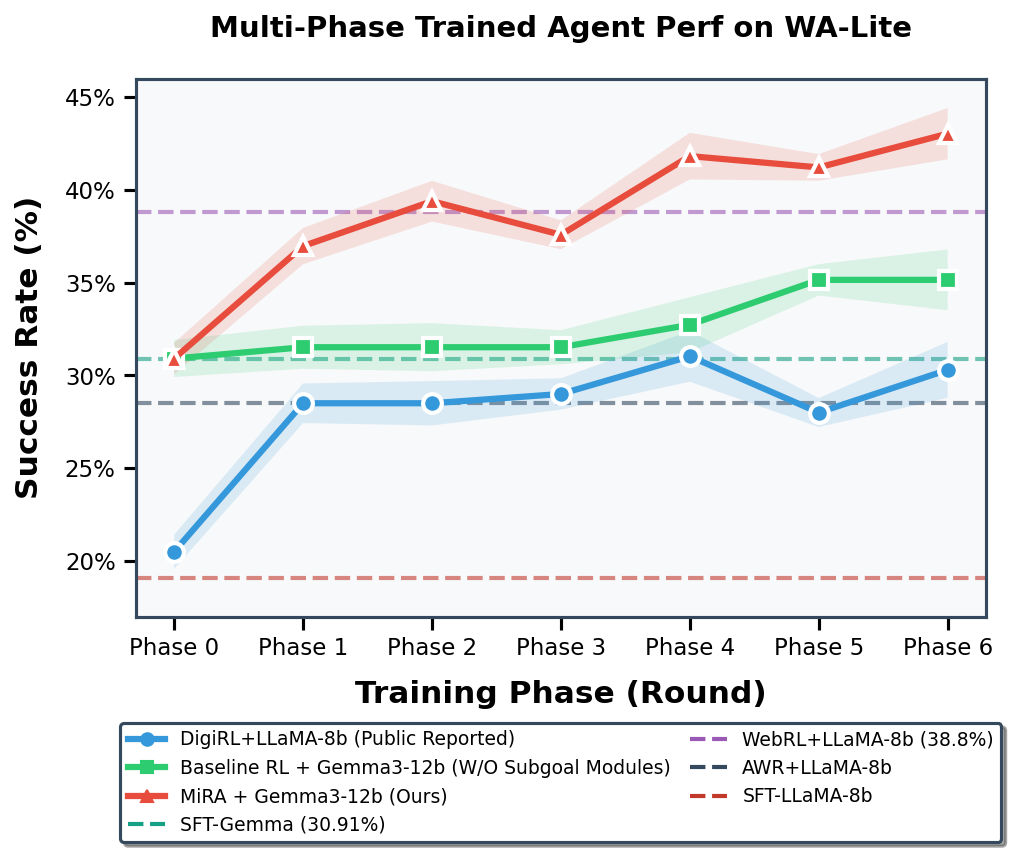}
        \caption{\textbf{Overall SR across Phases.} 
        MiRA (red) continually improves with each refinement phase, outperforming both standard RL (green) and previous methods such as WebRL and SFT.}
        \label{fig:sr_overall}
    \end{subfigure}
    \hfill
    \begin{subfigure}[t]{0.47\linewidth}
        \centering
        \includegraphics[width=\linewidth]{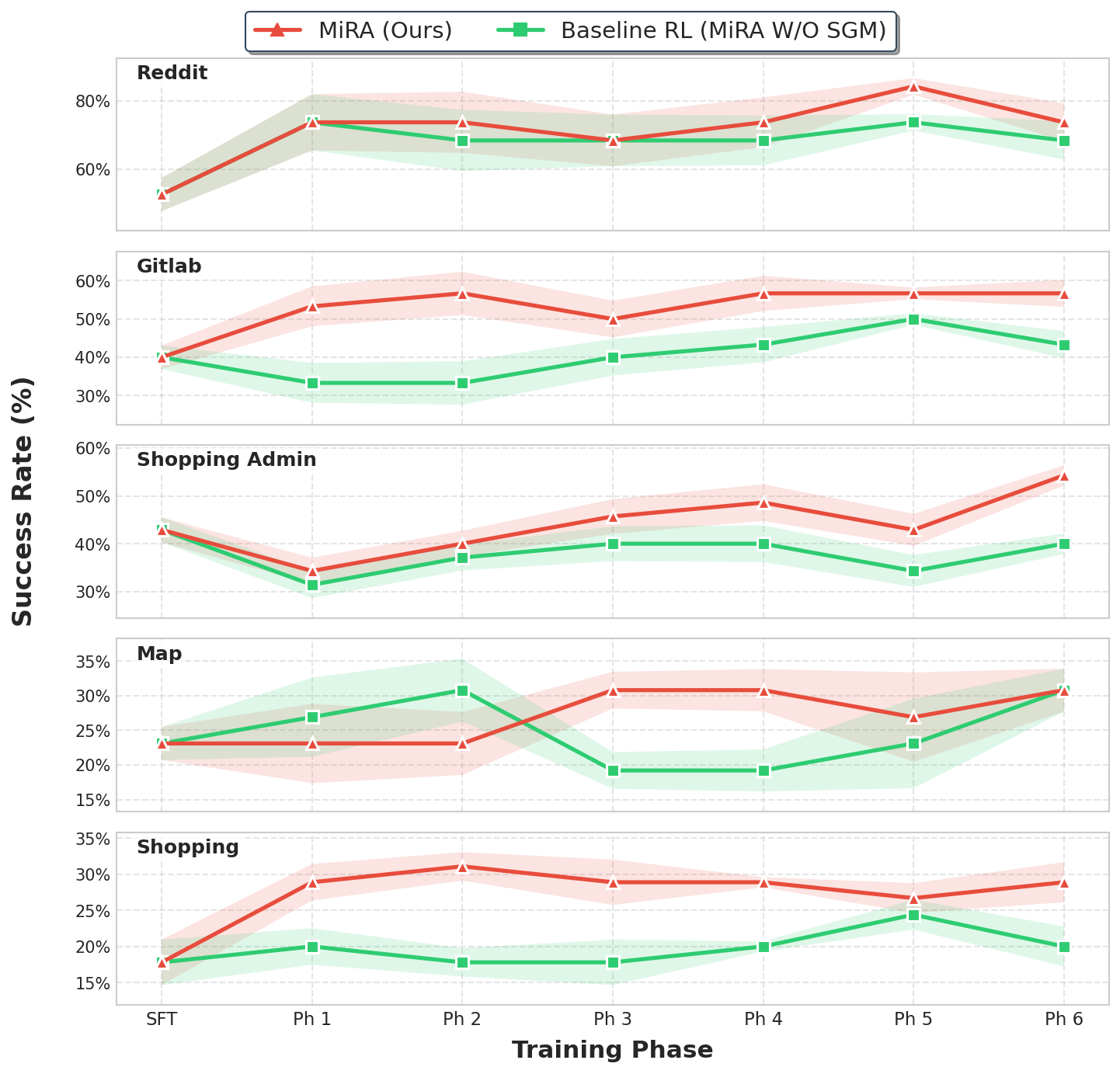}
        \caption{\textbf{Per-Site SR Trends.} 
        Domain-wise improvements of MiRA compared to baseline RL, 
        showing consistent gains across Reddit, Gitlab, CMS, Map, and Shopping Admin.}
        \label{fig:sr_site}
    \end{subfigure}
    \caption{\textbf{Multi-Phase and Domain-Level Success Rate Trends.}
    The left panel tracks the overall training progress over curriculum phases,  while the right panel breaks down SR improvements per website domain. 
    MiRA demonstrates stable, cumulative gains and broader generalization across domains. Results are averaged from 5 runs.}
    \label{fig:sr_combined}
\end{figure}

Across all training rounds, MiRA steadily improves from an initial SR of \textasciitilde31\% to \textbf{43\%}, while the baseline RL model without subgoal shaping saturates near 35\%. Per-domain analysis shows consistent upward trends, particularly on Gitlab and Shopping Admin, where subgoal-based shaping provides denser intermediate rewards for complex multi-step interactions. This confirms that MiRA's subgoal-oriented design not only enhances overall learning efficiency but also stabilizes long-horizon behaviors across diverse web environments. 

For each WebArena-Lite task, we evaluate the performance of:
\begin{itemize}
    \item \textbf{MiRA w/o Subgoal Modules} (Baseline RL),
    \item \textbf{MiRA} (Ours).
\end{itemize}

For a given task with $c$ successful rollouts out of $n$ total attempts, we compute \textit{pass@k} for $k \in \{1, 2, 4, 8\}$ using the standard unbiased estimator:
\begin{equation}
    \text{pass}@k = 1 - \frac{\binom{n-c}{k}}{\binom{n}{k}},
\end{equation}
where the calculation is valid for $k \le n$. This metric estimates the probability that at least one solution is correct given $k$ samples. We report the average \textit{pass@k} across all tasks for both the MiRA w/o Subgoal Modules and the full MiRA models to demonstrate the consistency improvements gained from our hierarchical approach. Results are reported in Figure~\ref{fig:pass_k_analysis}, MiRA consistently outperforms the baseline across all sample budgets ($k$). Notably, the performance gap widens significantly in Phase 2 ($+7.9\%$ at Pass@2), demonstrating faster convergence. While the baseline narrows the gap by Phase 6 due to curriculum effects, MiRA maintains a strong lead ($+7.5\%$ at Pass@8).

\subsubsection{Ablation Study on MiRA}

A natural question arises: \emph{Is the full MiRA framework truly necessary, or can simpler alternatives achieve comparable performance?} To answer this, we systematically ablate key components and compare against standard RL baselines. Figure~\ref{fig:ablation_components} summarizes the results across training phases.

\paragraph{Baselines and Ablated Variants.}
We compare five configurations: (1) MiRA (Full): the complete framework with MSE-based policy regression (Eq.~\ref{eq:mse_objective}), potential critic for dense shaping, and doubly-robust advantage estimation (Eq.~\ref{eq:V4A}); (2) MiRA (w/o PC): removes the potential critic $P_\psi$, relying solely on sparse outcome rewards; (3) MiRA (w.\ KL): replaces the MSE policy objective with KL-divergence minimization (Eq.~\ref{eq:kl_objective}); (4) MiRA (w/o Doubly Adv.): uses only the 1-step TD error for advantage estimation ($\lambda=1$ in Eq.~\ref{eq:V4A}), removing the Monte-Carlo return mixing; (5) AWR: Advantage Weighted Regression~\citep{peng2019advantage}, a simpler off-policy baseline.

\paragraph{Why MSE over KL Divergence?}
A key design choice in MiRA is the use of MSE regression on log-probability ratios (Eq.~\ref{eq:mse_objective}) rather than directly minimizing KL divergence between $\pi_\theta$ and the optimal policy $\pi^*$. When KL divergence is used, the optimization objective becomes:
\begin{equation}
\label{eq:kl_objective}
\arg\min_{\theta} \, \mathbb{E}_{s \sim d(s)} \left[ D_{\mathrm{KL}}\left( \pi^*(\cdot|s,I) \,\|\, \pi_\theta(\cdot|s,I) \right) \right]
= \arg\max_{\theta} \, \mathbb{E}_{s \sim d(s), a \sim \pi^*(a|s,I)} \left[ \log \pi_\theta(a|s,I) \right].
\end{equation}
This formulation imposes a critical constraint\footnote{see Appendix~\ref{sec:apdx_mse_vs_kl} for detailed formulation}: the training data must be sampled from $\pi^* \propto \pi_{\mathrm{ref}} \exp\!\left(\tfrac{1}{\beta}A^*\right)$. In practice, this means actions must be drawn from the reference policy $\pi_{\mathrm{ref}}$, severely limiting the ability to leverage diverse off-policy data from replay buffers. Furthermore, Eq.~\ref{eq:kl_objective} can only \emph{increase} the probability of sampled actions---it cannot explicitly \emph{decrease} the probability of actions with negative advantage. When beneficial actions are rare under $\pi_{\mathrm{ref}}$, the KL objective may paradoxically reinforce suboptimal behaviors.


\begin{figure*}[t]
    \centering
    \begin{subfigure}[b]{0.5\textwidth}
        \centering
        \includegraphics[width=0.9\linewidth]{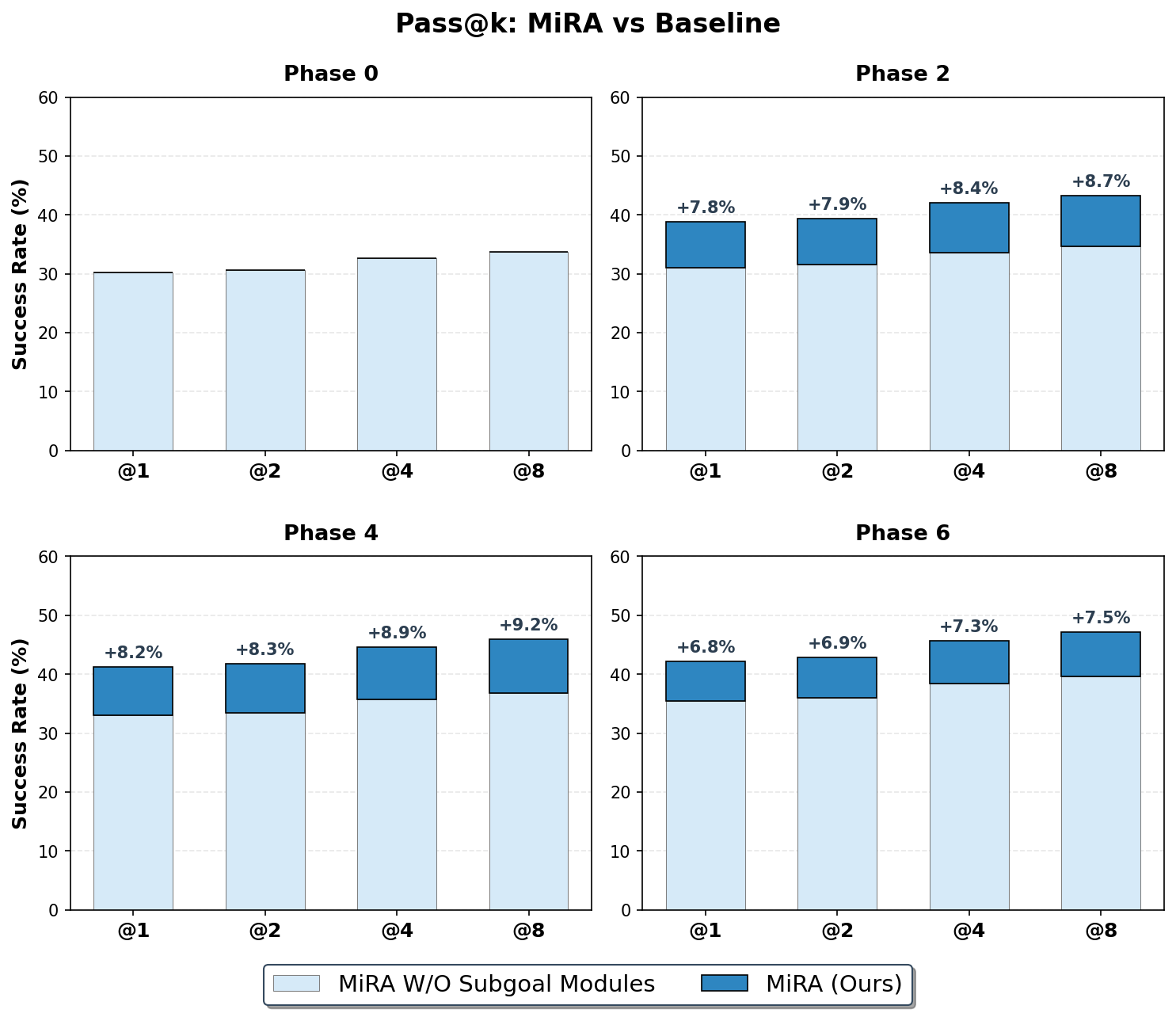}
        \caption{\textbf{Pass@$k$ Scaling.} Performance comparison across varying sample budgets ($k \in \{1, 2, 4, 8\}$) after different training stages.}
        \label{fig:pass_k_analysis}
    \end{subfigure}
    \hfill
    \begin{subfigure}[b]{0.48\textwidth}
        \centering
        \includegraphics[width=\linewidth]{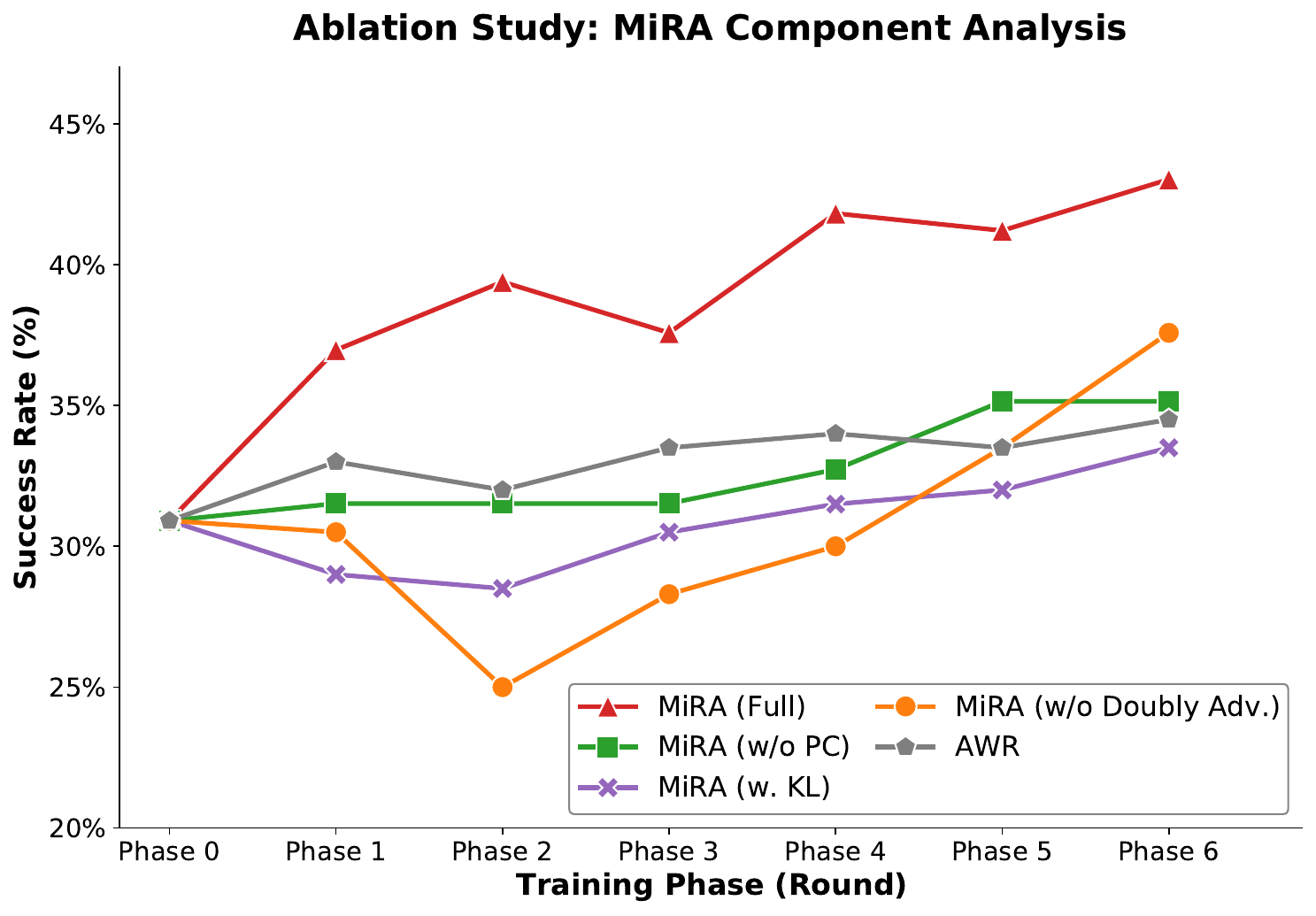}
        \caption{\textbf{Component Ablation.} Impact of removing the potential critic (PC), doubly-robust estimation, and MSE updates. All methods start from the same SFT checkpoint ($\text{SR} \approx 31\%$) and each result stands for the average number from 3 trials.}
        \label{fig:ablation_components}
    \end{subfigure}
    
    \caption{\textbf{Robustness and Component Analysis.} \textbf{(a)} MiRA consistently outperforms the baseline across all sample budgets ($k$). \textbf{(b)} Ablation results confirm that removing key components degrades success rates toward the SFT baseline ($\approx 31\%$), validating the necessity of the full MiRA framework.}
    \label{fig:combined_analysis}
    \vspace{-5pt}
\end{figure*}

In contrast, the MSE objective (Eq.~\ref{eq:mse_objective}) operates on any data distribution $\nu$ and directly regresses the log-ratio toward the advantage signal, enabling both upweighting of good actions and downweighting of bad ones. The purple curve (\textbf{MiRA w.\ KL}) in Figure~\ref{fig:ablation_components} confirms this analysis: it initially drops below the SFT baseline and recovers slowly, achieving only about $33\%$ by Phase~6---nearly $10\%$ below the full method.

\paragraph{The Critical Role of Doubly-Robust Advantage Estimation.}
The orange curve (\textbf{MiRA w/o Doubly Adv.}) reveals a striking failure mode: performance \emph{collapses} to around $25\%$ in early phases before gradually recovering. This occurs because the value critic $V_\phi$ is poorly calibrated at the start of training. Using only the 1-step TD error $r'_t + \gamma V_\phi(s_{t+1}) - V_\phi(s_t)$ propagates this bias directly into the advantage estimates, causing the policy to optimize toward incorrect targets. By Phase~6, as the critic improves, performance recovers to over $37\%$---but the early damage delays convergence significantly.

The doubly-robust estimator (Eq.~\ref{eq:V4A}) mitigates this by mixing TD estimates with Monte-Carlo returns $G_t$. The MC component provides unbiased (though higher variance) gradient signals that anchor learning even when the critic is unreliable, preventing the catastrophic early-phase degradation.

\paragraph{Dense Shaping via Potential Critic.}
The green curve (\textbf{MiRA w/o PC}) demonstrates the sparse reward problem: without the potential critic's dense shaping signal, learning plateaus at $\approx35\%$. The agent cannot effectively credit intermediate actions for eventual success, becoming trapped in local optima. The gap between the full method and this ablation confirms that subgoal-based potential shaping is essential for bridging the credit assignment gap in long-horizon web tasks.

\paragraph{Summary.}
The ablation results demonstrate that each component addresses a distinct failure mode: the MSE objective enables off-policy learning and bidirectional probability adjustment; the doubly-robust estimator prevents early-phase collapse from critic bias; and the potential critic provides the dense supervision necessary for long-horizon credit assignment. Removing any single component results in substantial performance degradation, validating the necessity of the full MiRA framework. Addtional parameter-sensitive abaltion studies can be found in Appendix~\ref{sec:apdx_ab}.

\subsubsection{Subgoal Completion Dynamics}

We further analyze the temporal evolution of subgoal completion to disentangle the agent's learning progress. Figure~\ref{fig:subgoal-pattern} visualizes the average completion probability for each subgoal across the episode horizon.

In Phase 0, the probability mass is heavily concentrated in the bottom-left region, forming a rigid vertical band over the first two subgoals (indices 0 and 1). This pattern indicates a severe \emph{early-stage stagnation}: the agent successfully initiates the task but fails to bridge the transition to intermediate objectives, effectively consuming its entire time budget in a local optimum without downstream progress. However, as training advances to Phase 6, we observe a decisive behavioral phase transition. The probability density shifts from this static vertical column to a structured \emph{diagonal gradient} extending from the top-left to the bottom-right. This strictly monotonic frontier demonstrates that the agent has acquired \emph{sequential fluency}: it no longer loiters on initial steps but efficiently chains subgoals in lockstep with the episode timeline. The emergence of this diagonal alignment confirms that MiRA has successfully broken the initial bottleneck, transforming fragmented, short-horizon attempts into coherent, long-horizon trajectories.

\begin{figure}[t!]
    \centering
    \includegraphics[width=1.0\linewidth]{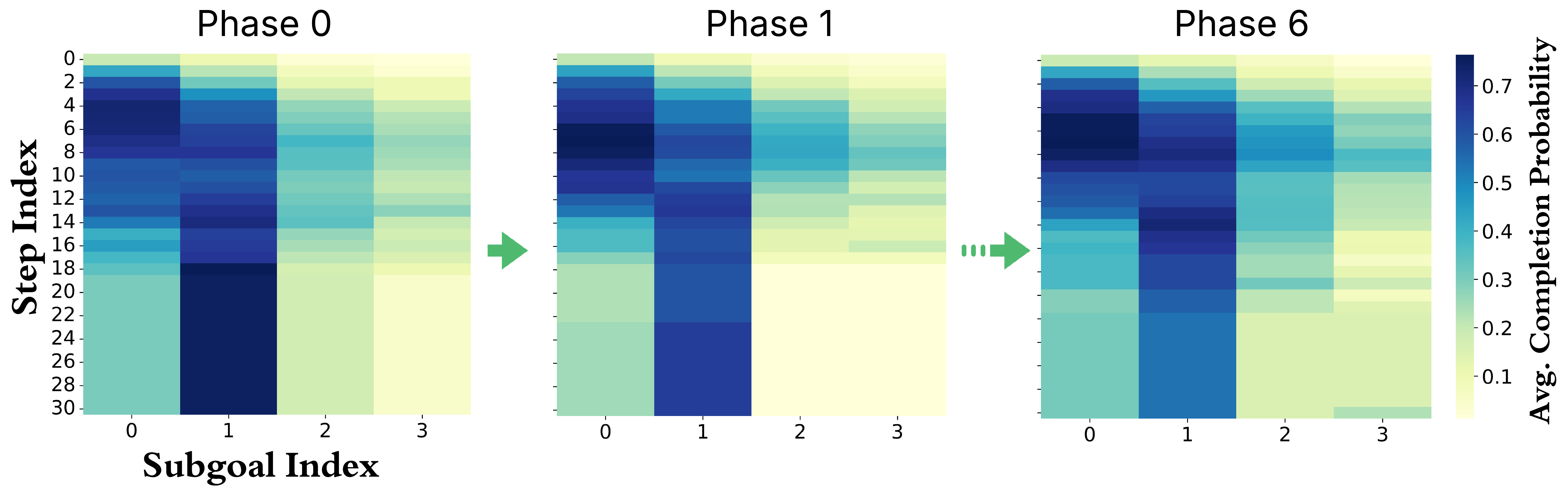}
    \caption{\textbf{Subgoal Completion Pattern.} 
    Average subgoal completion probabilities across time steps for consecutive training rounds (Phase 0 $\rightarrow$ Phase 1 $\rightarrow$ Phase 6). 
    Phase 0 shows stagnation at Subgoal 1, while Phase 6 exhibits a diagonal gradient, indicating that the agent has learned to sequentially chain subgoals over the episode horizon.}
    \label{fig:subgoal-pattern}
\end{figure}

\subsubsection{Result Analysis: Resolving the Planning Bottleneck}

To quantify the impact of MiRA on the agent's fundamental planning capabilities, we applied our automated failure analyzer to the post-training trajectories. Figure~\ref{fig:mira_failure_dist} presents the comparative distribution of failure modes across base LLMs, SFT baselines, the \textsc{WebRL} baseline, and our MiRA agent.

\begin{figure*}[t]
    \centering
    \includegraphics[width=1.0\linewidth]{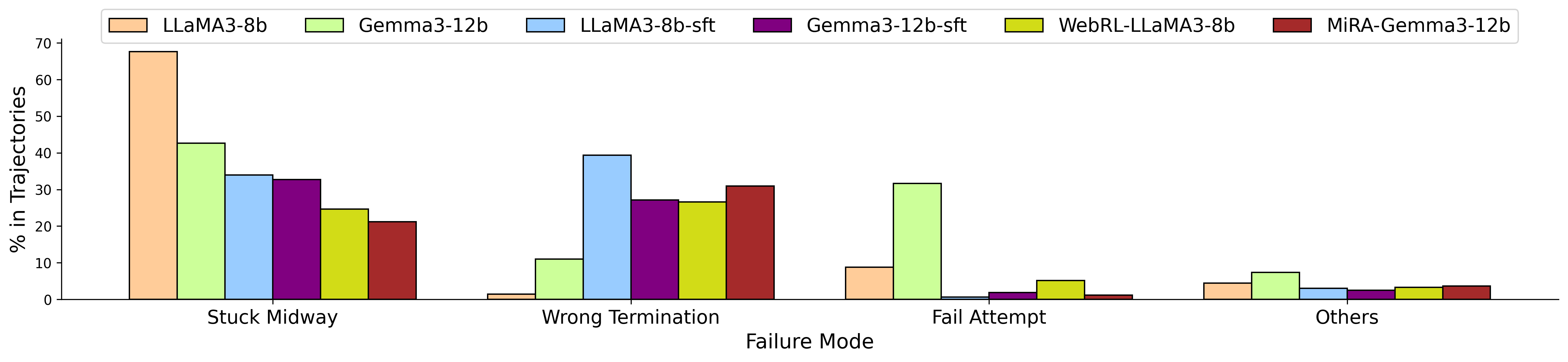}
    \caption{\textbf{Shift in Failure Distribution Post-Training.} Comparison of failure modes across Base, SFT, and RL-tuned models. MiRA (dark red) significantly reduces the ``Stuck Midway'' error rate compared to SFT and standard WebRL baselines, indicating that the potential-based curriculum successfully helps the agent break out of local optima and navigation loops.}
    \label{fig:mira_failure_dist}
\end{figure*}

The most significant finding is the drastic reduction in \emph{Get Stuck Midway} errors. As shown in Figure~\ref{fig:mira_failure_dist}, base models (e.g., LLaMA3-8b) and even their SFT counterparts suffer heavily from navigation loops, with Gemma3-12b-SFT exhibiting a stuck rate of $\approx 33\%$. In contrast, MiRA reduces this category to $\approx 21\%$, outperforming both the SFT baseline and the competitive \textsc{WebRL} method ($\approx 25\%$). This reduction serves as strong empirical evidence that our potential-based reward shaping and curriculum strategy effectively smoothen the optimization landscape. By providing dense subgoal signals, MiRA enables the agent to recover from repetitive states where standard policies typically stall, effectively ``pushing'' the agent through long-horizon dependencies.

Interestingly, while MiRA minimizes navigation stalls, we observe a relative increase in \emph{Wrong Termination} errors (rising to $\approx 31\%$). We note that this phenomenon is consistent across all trained agents: since successful termination sequences are highly rewarded during training, optimized policies naturally develop a bias toward attempting termination actions even without full confidence. Furthermore, the rate of wrong termination is compounded by the AutoRater (LLM-as-Judge), which judges success based on reaching the correct terminal page, often failing to verify the semantic details of the final answer. This shift, however, represents a clear progression in capability. Rather than failing to navigate (execution failure), the agent now traverses the full horizon to reach a terminal state, indicating that MiRA has successfully solved the lower-level planning bottleneck. This effectively exposes the higher-level semantic reasoning limitations of the underlying LLM—a preferable trade-off for scalable web automation compared to navigation stagnation.

\section{Discussion}

In this work, we addressed the critical planning bottleneck in autonomous web agents by introducing MiRA, a unified framework that synergizes explicit milestone-based reasoning with potential-based reinforcement learning. By transforming high-level intent into verifiable subgoals, we demonstrated that both proprietary inference engines and open-source RL policies can overcome the horizon truncation failures pervasive in current benchmarks. Beyond the empirical gains on WebArena-Lite, our findings open several avenues for discussion regarding the nature of reward shaping, the reliability of intermediate supervision, and the future of self-evolving autonomy.

\paragraph{Connection to Process Reward Models (PRM).}
Conceptually, our training methodology can be viewed as a form of semi-supervised Process Reward Model (PRM) learning. Unlike traditional PRMs that rely on expensive human annotations for every step of a chain-of-thought, MiRA synthesizes its own ``process supervision'' by generating semantic subgoals via an advanced teacher model and grounding them in execution traces. The potential critic $P_\psi$ effectively acts as a learned value function over these synthesized milestones, densifying the sparse outcome signal. This aligns with recent findings in mathematical reasoning \citep{lightman2023let}, suggesting that verifying intermediate steps (subgoals) is often easier and more sample-efficient than verifying final answers alone. Our work extends this principle to the noisy, non-stationary domain of web navigation, proving that automated process supervision is a viable path to scaling RL for complex agents.

\paragraph{The ``Solid Grounding'' Trade-off.}
The efficacy of our potential-based shaping hinges entirely on the validity of the generated subgoals. As shown in our ablation studies, rewarding \textbf{solid} subgoals—those the agent can robustly perceive and achieve—creates a reliable curriculum, forming a regression curve of progress (e.g., in Gitlab tasks) that pulls the agent deeper into the episode. However, this reliance introduces a subtle \emph{amplification bias}. Because the auxiliary reward is triggered only upon successful grounding of a milestone, the system heavily favors trajectories that master the early stages of a task. If an agent fails to ground the initial subgoals due to extreme exploration difficulty or perception errors, the shaping signal remains silent, effectively reverting the optimization to a sparse-reward regime. Thus, while MiRA excels at extending the ``competence boundary'' of agents that can start a task, it implies that future work must still address the cold start exploration problem in environments where even the first milestone is hard to reach.

\paragraph{Towards Self-Evolving Autonomous Agents.}
Perhaps the most compelling implication of our results is the feasibility of a fully self-contained, self-improving cycle. In our experiments, a single capable model (e.g., Gemini-2.5-pro) served multiple roles: it planned paths, executed actions, judged its own progress (via the \textit{AutoRater}), synthesized new curricula from its failures, and generated the shaping signals used to train its successor. This closed-loop design suggests that we are moving closer to truly \textbf{self-evolving agents} capable of recursive improvement. By integrating the planning, diagnosing, and training processes into a singular, cohesive loop, future systems could autonomously continually expand their capabilities without human intervention, transforming the static ``train-then-deploy'' paradigm into a lifelong learning process.

\paragraph{Future Directions.}
While our framework establishes a strong baseline, several dimensions warrant further exploration to generalize this approach. First, regarding milestone synthesis: we aim to transition from heuristic prompts to \emph{learnable or hierarchical subgoal generators}. Such models could dynamically tailor the granularity of milestones for knowledge-sparse domains where static linear ramps are insufficient. Second, we see promise in refining the shaping mechanism itself; future work should explore \emph{non-linear progress estimation} that account for the varying difficulty of distinct subgoals rather than treating them uniformly. Finally, we propose investigating \emph{signal annealing strategies}. In this regime, subgoals would serve as temporary ``warm-up'' scaffolding that is gradually withdrawn as training progresses, ensuring that the final policy relies on the true task objective rather than over-optimizing for auxiliary rewards.

\section{Conclusion}

We have presented a comprehensive framework that enhances the long-horizon planning capabilities of web agents through dynamic milestoning and potential-based reinforcement learning. Our approach yields state-of-the-art performance on the leading benchmark, boosting the open-source Gemma-3-12B model to outperform vastly larger proprietary systems. By rigorously identifying planning-oriented failures as the dominant error mode and resolving them through dense subgoal feedback, we provide both a robust engineering solution and a clear scientific insight into the limitations of current LLM agents. As web environments grow in complexity, we believe that such explicit, structured planning—embedded both in inference and optimization—will remain essential for bridging the gap between chat-based assistants and truly autonomous digital coworkers.

\section*{Author Contributions and Acknowledgments}

This work was conducted during the first author's internship within the Autonomous Agents (A2) team. Co-authors Sian Gooding and Edward Grefenstette provided overarching guidance, mentorship, and conceptual direction throughout the project. We especially thank Sian for her expertise in error and failure analysis, which shaped both the diagnostic methodology and the final framework. Oriana Riva and Florian Hartmann contributed essential insights to the problem formulation, experimental design, and data preparation, and offered expert feedback on GUI agent automation results.

Beyond the author contributions, we extend our sincere appreciation to Hao Sun and Martin Klissarov for their discussions on the core RL algorithm design and the deep understanding of the agentic training pipeline. We also thank the entire A2 team for providing the computing resources that made this work possible.


\bibliography{main}

\newpage
\appendix

\section{Appendix}

\subsection{Proof of PBRS Equivalence for Goal-Conditioned Tasks}
\label{sec:apdx_proof}

We provide the mathematical proof that Potential-Based Reward Shaping (PBRS) is equivalent to training with the original reward function but initializing the Q-function to the potential function $\Phi$. This proof holds for the goal-conditioned setting, where the reward, value, and potential functions all depend on a goal $g \in \mathcal{G}$.

\subsubsection{Equivalence of Bellman Updates}

Let's define two learning scenarios for a goal-conditioned Bellman optimality update.

\begin{itemize}
    \item 
    \textbf{Scenario A: PBRS (Shaped Reward)}
    \begin{itemize}
        \item Reward: $\tilde{r}(s, a, s', g) = r(s, a, s', g) + \gamma \Phi(s', g) - \Phi(s, g)$
        \item Q-function: $\tilde{Q}_k(s, a, g)$
        \item Initialization: $\tilde{Q}_0(s, a, g) = 0$
        \item Bellman Update: $\tilde{Q}_{k+1}(s, a, g) = \mathbb{E}_{s' \sim P(\cdot|s,a)} \left[ \tilde{r}(s, a, s', g) + \gamma \max_{a'} \tilde{Q}_k(s', a', g) \right]$
    \end{itemize}
    
    \item
    \textbf{Scenario B: $\Phi$-Initialization (Original Reward)}
    \begin{itemize}
        \item Reward: $r(s, a, s', g)$
        \item Q-function: $Q_k(s, a, g)$
        \item Initialization: $Q_0(s, a, g) = \Phi(s, g)$
        \item Bellman Update: $Q_{k+1}(s, a, g) = \mathbb{E}_{s' \sim P(\cdot|s,a)} \left[ r(s, a, s', g) + \gamma \max_{a'} Q_k(s', a', g) \right]$
    \end{itemize}
\end{itemize}

\begin{claim}[Wiewiora, 2003, adapted for goal-conditioning]
For any goal $g \in \mathcal{G}$ and all $k \ge 0$, the Q-functions from the two scenarios are related by:
$$ \tilde{Q}_k(s, a, g) = Q_k(s, a, g) - \Phi(s, g) $$
\end{claim}

\begin{proof}[Proof by Induction]
We prove this claim by induction on the iteration step $k$.

\textbf{Base Case (k=0):}
We check the relationship for the initial Q-values:
\begin{align*}
    \tilde{Q}_0(s, a, g) &= Q_0(s, a, g) - \Phi(s, g) \\
    0 &= \Phi(s, g) - \Phi(s, g) \\
    0 &= 0
\end{align*}
The base case holds.

\textbf{Inductive Step:}
Assume the relationship holds for an arbitrary step $k$. This is our inductive hypothesis:
$$ \tilde{Q}_k(s, a, g) = Q_k(s, a, g) - \Phi(s, g) $$
We must now prove that it also holds for $k+1$. We start with the Bellman update for Scenario A:
$$ \tilde{Q}_{k+1}(s, a, g) = \mathbb{E}_{s' \sim P(\cdot|s,a)} \left[ \tilde{r}(s, a, s', g) + \gamma \max_{a'} \tilde{Q}_k(s', a', g) \right] $$
Substitute the PBRS reward definition $\tilde{r}$:
$$ \tilde{Q}_{k+1}(s, a, g) = \mathbb{E}_{s'} \left[ (r(s, a, s', g) + \gamma \Phi(s', g) - \Phi(s, g)) + \gamma \max_{a'} \tilde{Q}_k(s', a', g) \right] $$
Now, substitute the inductive hypothesis $\tilde{Q}_k(s', a', g) = Q_k(s', a', g) - \Phi(s', g)$ into the $\max$ term:
$$ \tilde{Q}_{k+1}(s, a, g) = \mathbb{E}_{s'} \left[ (r + \gamma \Phi(s', g) - \Phi(s, g)) + \gamma \max_{a'} \left( Q_k(s', a', g) - \Phi(s', g) \right) \right] $$
The potential $\Phi(s', g)$ is a function of state $s'$ and goal $g$, but not action $a'$. Therefore, it is a constant with respect to the $\arg\max$ operation over $a'$ and can be pulled out:
$$ \max_{a'} \left( Q_k(s', a', g) - \Phi(s', g) \right) = \left( \max_{a'} Q_k(s', a', g) \right) - \Phi(s', g) $$
Substitute this back into our main equation:
\begin{align*}
    \tilde{Q}_{k+1}(s, a, g) &= \mathbb{E}_{s'} \left[ (r + \gamma \Phi(s', g) - \Phi(s, g)) + \gamma \left( \left( \max_{a'} Q_k(s', a', g) \right) - \Phi(s', g) \right) \right] \\
    &= \mathbb{E}_{s'} \left[ r + \gamma \Phi(s', g) - \Phi(s, g) + \gamma \max_{a'} Q_k(s', a', g) - \gamma \Phi(s', g) \right]
\end{align*}
The $\gamma \Phi(s', g)$ terms cancel out:
$$ \tilde{Q}_{k+1}(s, a, g) = \mathbb{E}_{s'} \left[ r(s, a, s', g) + \gamma \max_{a'} Q_k(s', a', g) - \Phi(s, g) \right] $$
Since $\Phi(s, g)$ depends on the current state $s$ and goal $g$, but not on the next state $s'$, we can pull it out of the expectation over $s'$:
$$ \tilde{Q}_{k+1}(s, a, g) = \left( \mathbb{E}_{s'} \left[ r(s, a, s', g) + \gamma \max_{a'} Q_k(s', a', g) \right] \right) - \Phi(s, g) $$
The term in the parentheses is exactly the Bellman update for $Q_{k+1}(s, a, g)$ from Scenario B.
$$ \tilde{Q}_{k+1}(s, a, g) = Q_{k+1}(s, a, g) - \Phi(s, g) $$
This completes the inductive step.
\end{proof}

\subsubsection{Implication: Goal-Conditioned Policy Invariance}

The result $\tilde{Q}_k(s, a, g) = Q_k(s, a, g) - \Phi(s, g)$ proves that the goal-conditioned Q-function learned with PBRS is simply the ``true'' goal-conditioned Q-function offset by the potential $\Phi(s, g)$. This directly proves policy invariance for any given goal $g$.

A greedy, goal-conditioned policy $\pi_k$ at step $k$ is derived by:
$$ \pi_k(s, g) = \arg\max_{a} Q_k(s, a, g) $$
If we derive the policy $\tilde{\pi}_k$ from the PBRS Q-function $\tilde{Q}_k$:
$$ \tilde{\pi}_k(s, g) = \arg\max_{a} \tilde{Q}_k(s, a, g) $$
Substitute our proven relationship:
$$ \tilde{\pi}_k(s, g) = \arg\max_{a} \left( Q_k(s, a, g) - \Phi(s, g) \right) $$
Since $\Phi(s, g)$ is a constant with respect to the $\arg\max$ over $a$ (as it does not depend on the action $a$), it does not change which action yields the maximum value.
$$ \arg\max_{a} \left( Q_k(s, a, g) - \Phi(s, g) \right) = \arg\max_{a} Q_k(s, a, g) $$
Therefore, for any state $s$ and goal $g$:
$$ \tilde{\pi}_k(s, g) = \pi_k(s, g) $$
The policy derived from the shaped Q-function is identical to the policy derived from the original Q-function at every step $k$. This guarantees that the final optimal policy $\pi^*(s, g)$ learned under the shaped reward $\tilde{R}$ is identical to the optimal policy under the original reward $R$.

To determine the scaling factor $\alpha$, we employ a validation-based selection strategy. We perform a grid search over $\alpha \in [0.1, 0.8]$ using a held-out set of validation tasks. The optimal $\alpha$ is selected to maximize the \emph{Success Rate (SR)}, subject to the constraint that the shaped reward magnitude does not overwhelm the ground-truth environmental signal. Empirically, we observed that $\alpha=0.3$ provided the best trade-off, effectively densifying the reward space while retaining sensitivity to the sparse terminal rewards.





\subsection{Benefits and Trade-offs of Online Reasoning via Subgoals}
\label{apdx:thinking_tradeoff}

While dynamic milestoning improves the grounding and reliability of agent reasoning, it also introduces an inherent trade-off between reasoning depth and computational efficiency. In particular, allocating additional ``thinking budget''—i.e., allowing the reasoning model to perform deeper reflective planning—can improve task success but increases inference latency.

\begin{wrapfigure}{r}{0.6\textwidth}
    \vspace{-10pt}
    \centering
    \includegraphics[width=0.9\linewidth]{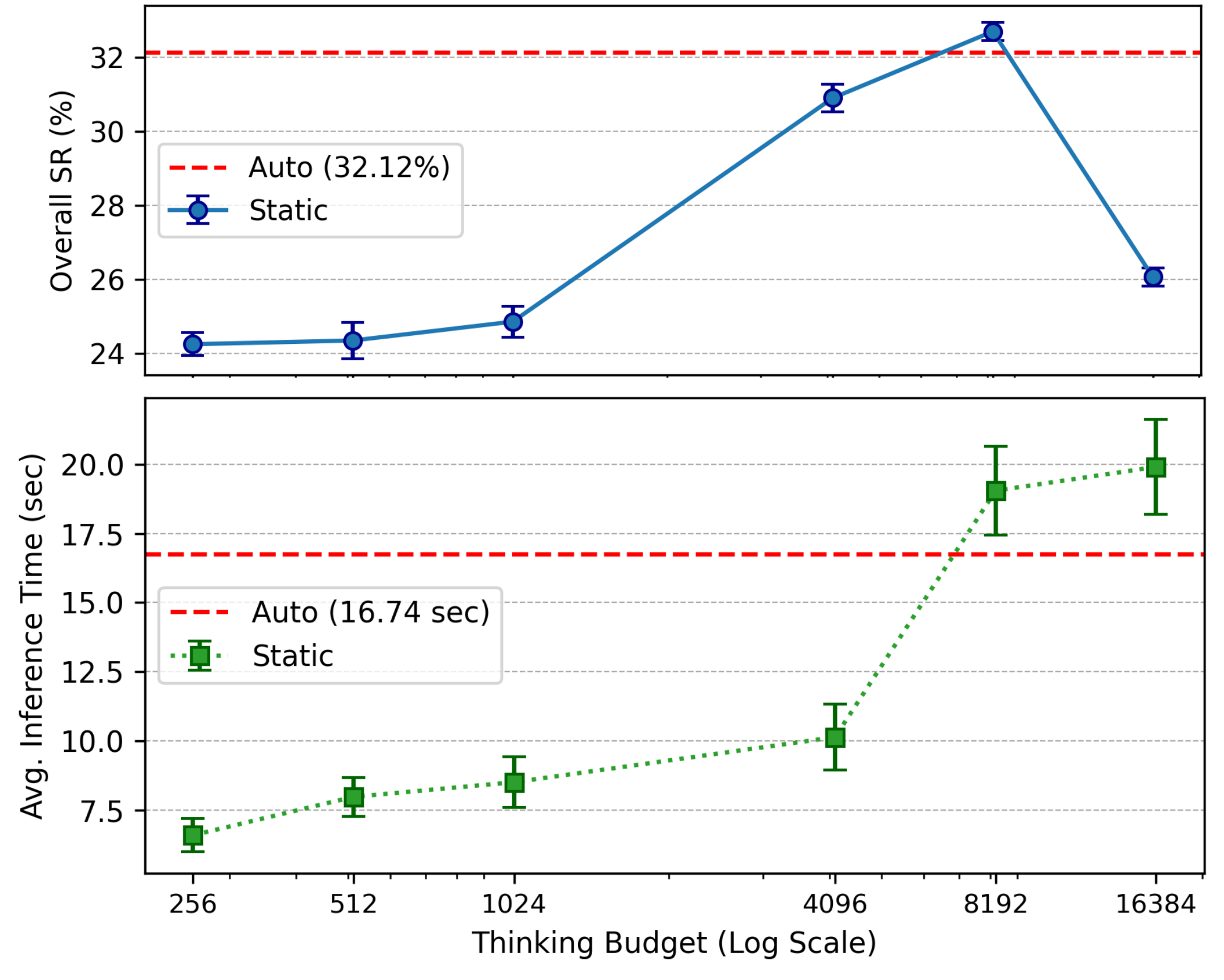}
    \caption{Overall Success Rate on WA-Lite and Inference Time vs. Thinking Budget. The ``Auto (Dynamic) Thinking'' mode from Gemini agent outperforms most static budgets, while static budgets show a clear trade-off between accuracy and rising inference time.\protect\footnotemark}
    \vspace{-10pt}
    \label{fig:thinking_tradeoff}
\end{wrapfigure}

\footnotetext{The inference times reported here are calculated from the averaged step-wise interactions across 25 identical tasks run at each budget level.}

Figure \ref{fig:thinking_tradeoff} presents the results from three repeated online experiments, illustrating the relationship between the allocated ``Thinking Budget'' and the resulting agent performance. Increasing the static thinking budget improves the Overall Success Rate (SR) from approximately $24.3\%$ (at 256 tokens) to a peak of approximately $32.5\%$ (at 8192 tokens). However, this improvement comes at the cost of a non-linear increase in inference time, which rises from approximately $6.5$ seconds to approximately $19$ seconds.

Notably, excessively large static budgets (e.g., 16384 tokens) lead to diminishing returns, with the success rate dropping to approximately $26\%$. This suggests that beyond a certain point, additional reasoning may introduce unnecessary deliberation without improving decision quality.

In comparison, the ``Auto (Dynamic)'' thinking mode used by the Gemini-based agent achieves a success rate of $32.12\%$, which is statistically comparable to the optimal static budget while reducing the average inference time to $16.74$ seconds. This dynamic allocation of reasoning effort enables the agent to balance accuracy and computational cost without requiring manual tuning of the thinking budget.

\subsection{Iterative Policy Refinement and Curriculum Generation}
\label{apdx:iterative_rl}

This section provides additional implementation details for the iterative refinement framework described in Section~\ref{sec:iterative_rl}. The full training pipeline alternates between environment interaction, offline RL optimization, and curriculum generation to progressively improve the agent.

\paragraph{Phase $K$: Environment Interaction and Data Collection.}
At the start of phase $K$, the current model $\text{RL-agent } K$ is deployed on the task distribution $\text{Task Set for Phase } K$ within the environment $\text{Env}$. This stage gathers diverse trajectories under the current policy's capabilities.

\begin{enumerate}
    \item \textbf{Rollout Execution.} The agent performs interactive rollouts within $\text{Env}$, producing sequences of states, actions, and observations that reflect both successful behaviors and failure cases.
    
\item \textbf{Trace Processing and Buffer Integration.} Newly generated rollouts are passed to the \emph{Trace Processing} module for annotation and filtering.
\begin{itemize}
    \item \textit{Subgoal Annotation:} We utilize LLM-based \textbf{subgoal checkers}\footnote{We adopt few-shot based Gemini-2.5-pro and engineered prompts} to verify the completion status of intermediate objectives at every step. These binary checks are interpolated to generate dense \textbf{subgoal progress scores} ($p_t^\ast$), which serve as regression targets for the Potential Critic.
    \item \textit{Perplexity Filtering:} To ensure the actor learns from high-quality data, we evaluate the \textbf{perplexity} of the current actor policy on the collected traces. Trajectories exhibiting perplexity above a dynamic threshold are discarded as out-of-distribution or irrational behavior.
    \item \textit{Training Target Assignment:} The filtered data is processed to support the hierarchical optimization in Algorithm~1. First, the \textbf{Potential Critic} is trained on the progress labels $p_t^\ast$. Second, the \textbf{Main Value Critic} is updated using the sparse terminal outcomes (ORM labels). Finally, using the updated potential landscape, we compute shaped rewards and the resulting \textbf{Monte-Carlo (MC) returns} ($G_t$) to optimize the \textbf{Actor}. For each phase, we perform one epoch of RL training over the collected data.
\end{itemize}
This produces an updated and temporally coherent \textbf{Training Data} corpus that captures the agent's evolving experience distribution while filtering out noise.
\end{enumerate}

\paragraph{Offline RL Training: Model Patching via Batched Updates.}
The aggregated $\text{Training Data}$ is then used to refine the model parameters through \emph{offline batched RL}, effectively ``patching'' the agent between deployment rounds. The training phase minimizes the regression-style policy loss $\mathcal{L}_\pi(\theta)$ from Eq.~\eqref{eq:mse_objective}, along with critic updates from Eqs.~\eqref{eq:potential_regression} and \eqref{eq:critic_value}. The result is a newly optimized policy $\text{RL-agent } K\!+\!1$, which inherits prior knowledge while incorporating corrections from the latest rollout data. Offline RL offers two practical benefits: (1) it stabilizes learning by decoupling environment noise from gradient updates, and (2) it enables fine-grained policy improvements at scale by batching large numbers of stored trajectories, including previously failed ones, into a consistent optimization phase.

\paragraph{Online Curriculum Generation: Failure-Driven Task Evolution.}
Once the offline phase concludes, the system performs targeted task selection to maintain progressive learning difficulty:

\begin{enumerate}
    \item \textbf{Task Generation ($\text{Task Gen.}$).} Instead of purely synthesizing new tasks, we utilize a pre-collected pool of 1,573 feasible human-selected tasks (disjoint from the evaluation set). The $\text{Task Gen.}$ module analyzes the $\text{Failure}$ traces from Phase $K$ and resamples tasks from this pool that exhibit high similarity to the failed instances. In practice, we use larger models such as Gemini-2.5-flash to identify semantic similarity, limit the top-$K$ similar tasks per failure trace, and allow overlap. For rare or unbalanced task categories where the pool coverage is sparse (e.g., ``Map'' navigation), we synthesize new instructions by perturbing failure samples to ensure balanced distribution.
    
    \item \textbf{Adaptive Phase Transition.} The resulting $\text{Task Set for Phase } K\!+\!1$ becomes the curriculum for the next phase. The refined agent $\text{RL-agent } K\!+\!1$ is then redeployed, closing the loop between environment feedback and policy improvement.
\end{enumerate}

\paragraph{Iterative Curriculum Refinement.}
This iterative process forms a continual adaptation loop:
\[
\text{Phase } K \;\xrightarrow[\text{Failures}]{\text{Task Gen.}}\; \text{Phase } K\!+\!1,
\]
where each cycle explicitly leverages failure cases as learning signals. Over successive phases, the curriculum automatically evolves from simple to complex tasks, while the policy progressively internalizes successful behaviors through potential-shaped, subgoal-oriented offline RL. This cyclical design ensures that MiRA-RL maintains long-horizon learning efficiency while continually repairing and extending the model's competence boundary.

\begin{algorithm}[t]
\caption{MiRA Outer Loop: Curriculum Generation \& Experience Replay}
\label{alg:mira_curriculum}
\begin{algorithmic}[1]

\State \textbf{Input:} Initial set $I_0$, task pool $I_{\text{pool}}$, SFT model $\mathcal{M}_{\text{SFT}}$, Envs, Reward Model $\mathcal{M}_{\text{ORM}}$
\State \textbf{Initialize:} $\pi_1 \gets \mathcal{M}_{\text{SFT}}$, Critics $V_\phi, P_\psi$, Buffer $\mathcal{B} \gets \emptyset$, Failure Set $\mathcal{F} \gets \emptyset$, $I_{\text{train}} \gets I_0$

\For{Phase $k = 1$ to $N$}
    \State Collect trajectories $D_{\text{roll}}$ via $\pi_k$ on $I_{\text{train}}$; label successes with $\mathcal{M}_{\text{ORM}}$
    \State Split $D_{\text{roll}} \to D_{\text{success}} \cup D_{\text{fail}}$
    \State Update $\mathcal{B} \gets \mathcal{B} \cup D_{\text{success}}$ and $\mathcal{F} \gets \mathcal{F} \cup \{\text{instructions of } D_{\text{fail}}\}$

    \State Calculate $\text{ppl}(\tau)$ for $\tau \in \mathcal{B}$ using current $\pi_k$
    \State Filter $D_{\text{exp}} \gets \big\{\tau \in \mathcal{B} \mid 1/0.9 \le \text{ppl}(\tau) \le 1/0.5 \big\}$
    \State Combine data: $D_{\text{train}} \gets D_{\text{roll}} \cup D_{\text{exp}}$

    \State $(\pi_{k+1}, V_\phi, P_\psi) \gets \textsc{MiRA-RL}(D_{\text{train}}, V_\phi, P_\psi, \pi_k, \dots)$

    \State $I_{\text{sim}} \gets \text{Resample}(I_{\text{pool}})$ based on semantic similarity to $\mathcal{F}$
    \State $I_{\text{syn}} \gets \text{Generate}(\mathcal{F})$ for rare/unbalanced failure types
    \State $I_{\text{train}} \gets \text{Select}(I_{\text{sim}} \cup I_{\text{syn}})$
\EndFor

\end{algorithmic}
\end{algorithm}

\subsection{Experimental Details}
\label{sec:apdx_exp}

In this section, we provide a detailed breakdown of the experimental setup, including the specific prompts used for subgoal decomposition, the hyperparameters for all baseline comparisons, and the system instructions that govern agent behavior and evaluation.

\paragraph{Subgoal Generation Examples and Decomposition}

\label{sec:apdx_subgoaleg}

To bridge the gap between high-level user intents and low-level browser actions, MiRA employs a multimodal decomposition strategy. We utilize \textbf{Gemini-2.5-pro} as the teacher model due to its strong visual reasoning capabilities. As shown in Figure~\ref{fig:subgoal-prompt}, the model is prompted with the task instruction and the initial webpage screenshot, and is strictly instructed to decompose the task into a fixed number of milestones (e.g., exactly 4 steps). This standardization ensures consistent temporal granularity across diverse tasks. Table~\ref{tab:subgoal_examples} presents qualitative examples of these generated subgoals across various WebArena environments (Reddit, GitLab, CMS, etc.), demonstrating the model's ability to identify logical checkpoints such as ``Maps to subreddit,'' ``Search for user,'' and ``Submit post.''

\renewcommand{\arraystretch}{1.1}

\newcommand{\tablist}[1]{%
    \begin{enumerate}[leftmargin=*, nosep, before=\vspace{2pt}, after=\vspace{2pt}]
        #1
    \end{enumerate}%
}

\begin{longtable}{p{0.12\textwidth} p{0.30\textwidth} p{0.52\textwidth}}
\captionsetup{justification=centering}
\caption{Examples of generated subgoals. See Figure~\ref{fig:agent_architecture} for online usage.}
\label{tab:subgoal_examples} \\
\toprule
\footnotesize \textbf{Website} & \footnotesize \textbf{User Intent} & \footnotesize \textbf{Generated Subgoals} \\
\midrule
\endfirsthead

\multicolumn{3}{c}%
{{\bfseries \footnotesize \tablename\ \thetable{} -- continued from previous page}} \\
\toprule
\footnotesize \textbf{Website} & \footnotesize \textbf{User Intent} & \footnotesize \textbf{Generated Subgoals} \\
\midrule
\endhead

\midrule
\multicolumn{3}{r}{{\footnotesize Continued on next page}} \\
\bottomrule
\endfoot

\bottomrule
\endlastfoot

\footnotesize

\multirow{8}{*}{\textbf{Reddit}} 
& Ask for product recommendations for running shoes within a budget of \$100 in r/sports 
& \tablist{
    \item Navigate to the ``r/sports'' subreddit.
    \item Initiate the creation of a new post.
    \item Write a title and body asking for shoes under \$100.
    \item Submit the post.
} \\ 
\cmidrule{2-3}
& Like all submissions created by CameronKelsey in subreddit earthporn 
& \tablist{
    \item Navigate to the subreddit `earthporn'.
    \item Search for posts by user `CameronKelsey'.
    \item Click the `like` button on the first post.
    \item Repeat for all other visible posts.
} \\
\midrule

\multirow{9}{*}{\textbf{Gitlab}} 
& Create a new public project ``awesome-llms'' and add primer, convexegg, abishek as members 
& \tablist{
    \item Navigate to new project creation page.
    \item Enter ``awesome-llms'', set visibility to public.
    \item Finalize project creation.
    \item Add primer, convexegg, and abishek in members page.
} \\
\cmidrule{2-3}
& Display issues in \texttt{kkroening and ffmpeg-python} with labels related to questions 
& \tablist{
    \item Navigate to \texttt{kkroening/ffmpeg-python} repo.
    \item Go to the ``Issues'' tab.
    \item Open ``Labels'' filter menu.
    \item Select label related to ``questions''.
} \\
\midrule

\multirow{8}{*}{\makecell[tl]{\textbf{Shopping}\\\textbf{Admin}\\\textbf{(CMS)}}} 
& Tell me the grand total of invoice 000000002 
& \tablist{
    \item Navigate to sales/invoices section.
    \item Search for invoice 000000002.
    \item Open details for the invoice.
    \item Identify grand total amount.
} \\
\cmidrule{2-3}
& Tell me the status of my latest order and when will it arrive 
& \tablist{
    \item Navigate to ``My Account'' or ``Profile''.
    \item Go to ``Orders'' or ``Order History''.
    \item Identify most recent order.
    \item Extract status and estimated arrival date.
} \\
\midrule

\multirow{9}{*}{\textbf{Map}} 
& Min travel time by car from Animal Rescue League of Pittsburgh to Schenley park? 
& \tablist{
    \item Navigate to map website, access directions.
    \item Enter ``Animal Rescue League'' as start.
    \item Enter ``Schenley park'' as destination.
    \item Select car mode, find min travel time route.
} \\
\cmidrule{2-3}
& Tell me the coordinates of Apple Store near Pitt in DD format 
& \tablist{
    \item Search for ``Pitt'' on map website.
    \item Search for ``Apple Store'' nearby.
    \item Select the correct Apple Store.
    \item Obtain coordinates in DD format.
} \\
\midrule

\multirow{9}{*}{\makecell[tl]{\textbf{Shopping}\\\textbf{(OSS)}}} 
& Look up most recent models of XBox controllers released 2020-2021? 
& \tablist{
    \item Search `Xbox controller' on Wikipedia.
    \item Navigate to main page.
    \item Locate `Models' or `History' section.
    \item Identify controllers released 2020-2021.
} \\
\cmidrule{2-3}
& Show least expensive switch card holder with capacity of 15+ cards. 
& \tablist{
    \item Search for ``switch card holder''.
    \item Filter for capacity of 15 or more.
    \item Sort results by price (low to high).
    \item Select first item.
} \\

\end{longtable}

\paragraph{Hyperparameter Configuration}
To ensure a fair comparison, we align our hyperparameters with established baselines wherever possible. Table~\ref{tab:hyperparams} lists the training configurations for SFT, AWR, DigiRL, WebRL, and MiRA. For the actor and critic networks, we maintain a consistent learning rate of $1 \times 10^{-6}$ and a batch size of 128 across all RL methods. A key distinction for MiRA is the \textbf{Potential Critic}, which uses a higher learning rate of $2 \times 10^{-5}$ and is trained for 3 epochs. This adjustment is necessary because the Potential Critic performs a regression task (fitting dense progress scores) rather than standard value estimation, requiring stronger gradient updates to capture the subtle shape of the progress landscape effectively.

\begin{table}[t!]
\centering
\footnotesize
\captionsetup{justification=centering}
\caption{Hyperparameters used in baseline methods.}
\label{tab:hyperparams}
\begin{tabular}{l l c}
\toprule
\textbf{Method} & \textbf{Hyperparameter} & \textbf{Value} \\
\midrule
\multirow{6}{*}{SFT} 
& learning rate & $1\times10^{-5}$ \\
& lr scheduler type & cosine \\
& warmup ratio & 0.1 \\
& batch size & 128 \\
& training epoch & 1 \\
& cutoff length & 16384 \\
\midrule
\multirow{8}{*}{AWR}
& actor learning rate & $1\times10^{-6}$ \\
& actor lr scheduler type & constant \\
& critic learning rate & $1\times10^{-6}$ \\
& critic lr scheduler type & constant \\
& batch size & 128 \\
& discount factor & 0.9 \\
& actor training epoch & 1 \\
& critic training epoch & 1 \\
\midrule
\multirow{12}{*}{DigiRL}
& actor learning rate & $1\times10^{-6}$ \\
& actor lr scheduler type & constant \\
& critic learning rate & $1\times10^{-6}$ \\
& critic lr scheduler type & constant \\
& instruction value function lr & $1\times10^{-6}$ \\
& instr.\ value function scheduler & constant \\
& batch size & 128 \\
& discount factor & 0.9 \\
& actor training epoch & 1 \\
& critic training epoch & 1 \\
& instruction value function epoch & 1 \\
& rollout temperature & 1 \\
& replay buffer size & 100{,}000 \\
\midrule
\multirow{8}{*}{WebRL}
& actor learning rate & $1\times10^{-6}$ \\
& actor lr scheduler type & constant \\
& critic learning rate & $1\times10^{-6}$ \\
& critic lr scheduler type & constant \\
& batch size & 128 \\
& discount factor & 0.9 \\
& actor training epoch & 1 \\
& critic training epoch & 1 \\
& rollout temperature & 1 \\
\midrule
\multirow{12}{*}{MiRA}
& actor learning rate & $1\times10^{-6}$ \\
& actor lr scheduler type & constant \\
& critic learning rate & $1\times10^{-6}$ \\
& critic lr scheduler type & constant \\
& potential critic learning rate & $2\times10^{-5}$ \\
& potential critic lr scheduler type & constant \\
& potential critic training epoch & 3 \\
& batch size & 128 \\
& discount factor & 0.9 \\
& actor training epoch & 1 \\
& critic training epoch & 1 \\
& rollout temperature & 1 \\
\bottomrule
\end{tabular}
\end{table}

\paragraph{System Prompts for Execution and Evaluation}
The performance of web agents is heavily influenced by the precise formatting of their inputs and allowed outputs. Figure~\ref{fig:apdx_agent} visualizes the data flow, where the agent receives the task instruction, action history, and current HTML. To constrain the agent's output to valid executable actions, we employ the system prompt detailed in Figure~\ref{fig:system-prompt} for small open-weight models (e.g., Llama-3, Gemma). This prompt restricts responses to a predefined Python-style DSL (e.g., \texttt{do(...)}, \texttt{click(...)}) and explicitly forbids the hallucination of non-visible elements.

\begin{figure}[t!]
    \centering
    \includegraphics[width=1.0\linewidth]{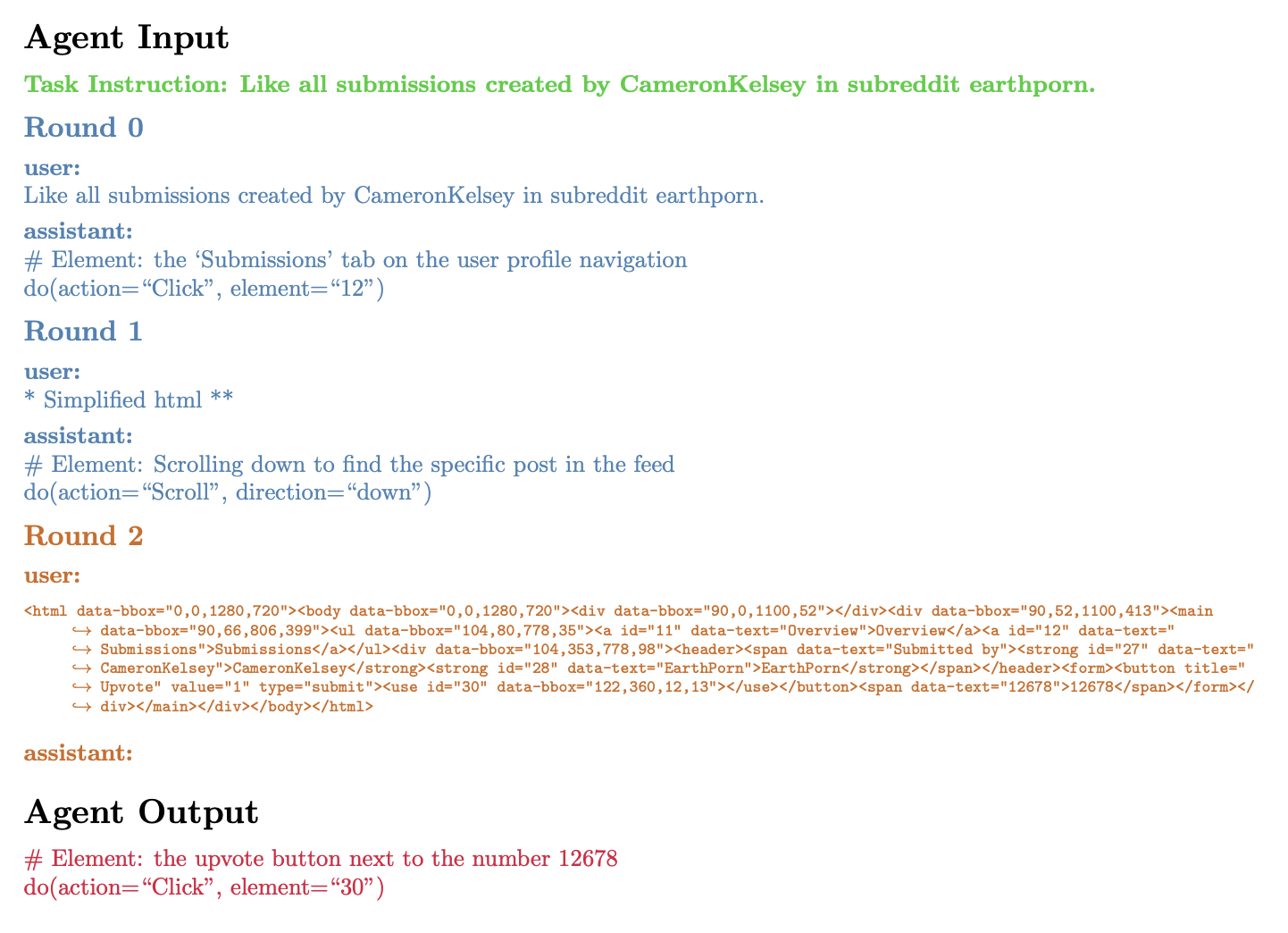}
    \caption{The input and output format of MiRA and baselines, where the input is composed of task instruction (in green), action history (in blue), and HTML of the current webpage (in orange). The output (in red) is the action taken on the current webpage.}
    \label{fig:apdx_agent}
\end{figure}

In contrast, when evaluating large proprietary models (e.g., GPT-4o, Gemini-1.5-Pro), we adopt the standard text-based accessibility tree prompt shown in Figure~\ref{fig:proprietary-prompt}. This configuration aligns with the default WebArena benchmark settings and prior works~\citep{chae2025web,qi2024webrl,wei2025webagent}, utilizing the stronger instruction-following capabilities of frontier models to parse structured element lists rather than simplified HTML. Finally, for reward modeling, we utilize the Outcome Reward Model (ORM) prompt shown in Figure~\ref{fig:orm-prompt}. This prompt instructs the judge to evaluate success based on the alignment between the user intent and the final page state, serving as the ground-truth signal for filtering successful trajectories during the warmup phase.

\begin{figure}[t!]
\centering
\footnotesize
\begin{tcolorbox}[colback=gray!3,colframe=black!40,title=System Prompt Used in small model baselines incl. MiRA]
\textbf{\# Setup} \\
You are a professional web browsing agent assistant that can fulfill user's high-level
instructions. Given simplified HTML of the browsed webpage at each step, you plan operations
in Python-style pseudo code using provided functions, or customize functions if necessary. \\[3pt]

\textbf{\# Predefined Functions}
\begin{verbatim}
def do(action, argument, element):
    """A single browsing operation on the webpage."""
def exit(message):
    """End the browsing process."""
def go_backward():
    """Go back to the previous page."""
def go_forward():
    """Go forward to the next page."""
\end{verbatim}

\textbf{Allowed actions:}
Click, Right Click, Type, Search, Hover, Scroll Up, Scroll Down, Press Enter,  
Switch Tab, Select Dropdown Option, Wait. \\[3pt]

\textbf{Examples:}
\begin{verbatim}
# Element: the `REPORTS' section on the left sidebar
do(action="Click", element="7")

# Element: the `Period' dropdown
do(action="Select Dropdown Option", argument="Month", element="20")
\end{verbatim}

\textbf{Rules:}
\begin{itemize}[nosep,leftmargin=10pt]
\item Only \textbf{one line of code} at a time.
\item Do not invent non-visible elements.
\item Use ``\# Element'' comments to indicate the chosen element.
\item Avoid loops; use if-else only when necessary.
\item If the page is loading, prefer ``Wait''.
\item Never use the browser address bar to navigate.
\end{itemize}

\textbf{User Prompt Template:}
\begin{verbatim}
Task Instruction: {instruction}
Action History: {action history}
The Current HTML: {html of last state}
\end{verbatim}
\end{tcolorbox}
\captionsetup{justification=centering}
\caption{The simplified system prompt employed in baseline agents.}
\label{fig:system-prompt}
\end{figure}

\begin{figure}[t!]
\centering
\footnotesize
\begin{tcolorbox}[colback=gray!3,colframe=black!40,title=System Prompt for Outcome Reward Model (ORM)]
\textbf{\# Role} \\
You are an expert in evaluating the performance of a website navigation agent. The agent is designed to help a human user navigate the website to complete a task. Given the user's intent, the agent's action history, and the final state of the screen, your goal is to decide whether the agent's execution is successful or not. You must respond with YES or NO. \\[3pt]

\textbf{\# User Prompt Template}
\begin{verbatim}
The User Intent:
{task instruction}

Action History:
{history of agent actions}

The Current Screenshot or Page:
{simplified html/screenshots of final state}
\end{verbatim}
\end{tcolorbox}
\caption{The system prompt used for the Outcome Reward Model (ORM) to verify trajectory success.}
\label{fig:orm-prompt}
\end{figure}

\begin{figure}[t!]
\centering
\footnotesize
\begin{tcolorbox}[colback=gray!3,colframe=black!40,title=System Prompt for Proprietary Models (GPT-4 / Gemini)]
\textbf{\# Role} \\
You are an autonomous intelligent agent tasked with navigating a web browser. You will be given web-based tasks and must accomplish them by issuing specific actions. \\[3pt]

\textbf{\# Input Information}
\begin{itemize}[nosep,leftmargin=10pt]
\item \textbf{Objective:} The user's task description.
\item \textbf{Observation:} A simplified accessibility tree representation of the webpage.
\item \textbf{Current URL:} The page you are currently navigating.
\item \textbf{Previous Action:} The action you just performed.
\end{itemize}

\textbf{\# Available Actions}
\begin{verbatim}
click [id]                  # Click on element with specific ID
type [id] [content] [0|1]   # Type content (0=no enter, 1=enter)
hover [id]                  # Hover over element
press [key_comb]            # Simulate key combination (e.g., "Ctrl+v")
scroll [up|down]            # Scroll the page
new_tab, close_tab          # Manage tabs
tab_focus [index]           # Switch to tab by index
goto [url]                  # Navigate to URL
go_back, go_forward         # History navigation
stop [answer]               # Complete task (with optional answer)
\end{verbatim}

\textbf{\# Response Format}
Your response must always start with the phrase ``In summary, the next action I will perform is'' followed by the action inside triple backticks.

\textbf{\# Examples}
\begin{verbatim}
OBSERVATION:
[1744] link `HP Inkjet Fax Machine'
[1757] button `Add to Cart'
URL: http://onestopmarket.com
OBJECTIVE: What is the price of the Fax Machine?

RESPONSE:
In summary, the next action I will perform is ```stop [$279.49]'''
\end{verbatim}

\begin{verbatim}
OBSERVATION:
[164] textbox `Search'
[171] button `Go'
URL: http://openstreetmap.org
OBJECTIVE: Show me hotels in Pittsburgh

RESPONSE:
In summary, the next action I will perform is
```type [164] [hotels in Pittsburgh] [1]'''
\end{verbatim}
\end{tcolorbox}
\caption{The standard system prompt (Set-of-Marks Style) used for large proprietary models (e.g., GPT-4, Gemini) in WebArena, utilizing a text-based accessibility tree observation space.}
\label{fig:proprietary-prompt}
\end{figure}

\begin{figure}[t!]
\centering
\footnotesize
\begin{tcolorbox}[colback=gray!3,colframe=black!40,title=Prompt for Subgoal Generation]
\textbf{\# Role} \\
You are an assistant inside the WebArena benchmark environment. WebArena simulates websites such as GitLab, Wikipedia, shopping portals, and admin dashboards. Tasks are solved only through browser interactions (clicks, typing, navigation)—never by recalling outside world knowledge. Your job is to decompose each task intent into EXACTLY 4 concrete, high-level subtasks. Always provide exactly 4 numbered steps, no more and no less. \\[3pt]

\textbf{\# Example 1: GitLab}
\begin{verbatim}
Task: Create a private JEKYLL repository called 'awesome_project'.
Screenshots: [Image1], [Image2], [Image3]

Answer:
1. Access GitLab and locate 'Create New'
2. Select the JEKYLL template from the template menu
3. Enter repo name 'awesome_project' and set it private
4. Create repository and confirm landing on its homepage
\end{verbatim}

\textbf{\# Example 2: Map/Wiki}
\begin{verbatim}
Task: Find the page of the college(s) where *The Chair* was filmed in
Pennsylvania other than those in Pittsburgh on the map.
Screenshots: [Image1], [Image2], [Image3]

Answer:
1. Identify correct Wikipedia page for *The Chair* (2021 TV series) and
   locate filming colleges in Pennsylvania (excluding Pittsburgh)
2. Copy the college name from the Wikipedia page
3. Open the map website and search the copied name
4. Navigate to the college's location on the map
\end{verbatim}

\textbf{\# User Prompt Template}
\begin{verbatim}
=====Your Turn=====
Q: TASK: {task_intent}
Screenshot: {task_initial_screenshot}
Decompose into exactly 4 milestones.
A: <SUBGOALS>
\end{verbatim}
\end{tcolorbox}
\caption{The multimodal prompt used to decompose tasks into four structured subgoals (Taking fixed 4 subgoals in total as example) via Gemini-2.5-pro. Prompt and examples used are optimized based on our discussions in section~\ref{sec:sg_gen}.}
\label{fig:subgoal-prompt}
\end{figure}

\subsection{MiRA Agent Design and Rationale}

\subsubsection{Observation and Action Space}
The MiRA agent operates within a structured environment designed to bridge the gap between high-level user instructions and low-level DOM interactions. The agent's observation space consists of three key components:
\begin{itemize}
    \item \textbf{User Instruction:} The natural language goal provided by the user.
    \item \textbf{Action History:} A chronological record of past actions, providing essential context for tracking subgoal progress.
    \item \textbf{Webpage HTML:} A simplified representation of the current webpage's HTML. To facilitate precise manipulation, we prune the DOM structure and assign unique element IDs to all interactive candidates.
\end{itemize}

The action space is designed to be granular yet comprehensive, enabling the execution of complex subgoals. The defined actions include:
\begin{itemize}[noitemsep]
    \item \textbf{Interaction:} \textit{Click}, \textit{Hover}, \textit{Type}, \textit{Search} (type and enter), \textit{Select dropdown option}.
    \item \textbf{Navigation:} \textit{Goto}, \textit{Go back}, \textit{Go forward}, \textit{Scroll}.
    \item \textbf{Tab Management:} \textit{New tab}, \textit{Tab focus}, \textit{Close tab}.
    \item \textbf{System:} \textit{Press} (key combinations), \textit{Exit} (terminate and respond).
\end{itemize}

To enhance interpretability and grounding, the agent appends metadata to its actions: a \texttt{\# Element:} tag describing the target DOM element, and a \texttt{\# Note:} tag citing supporting information from the webpage content.

\subsubsection{Training Framework and Rationale}
The design of the MiRA training pipeline addresses the challenge of learning robust policies from sparse rewards. The process follows a self-evolving curriculum:

\paragraph{Initialization via Supervised Fine-Tuning (SFT):}
We first establish a baseline capability by performing SFT on the agent using the WebArena-Lite dataset. This SFT-trained model initializes both the actor and the critic (with a randomly initialized value head) and is used to populate the initial replay buffer and failure set.

\paragraph{Self-Evolving Curriculum Reinforcement Learning:}
To progressively enhance performance, we employ a curriculum-based reinforcement learning approach. In each phase, we sample up-to 400 new instructions from the full task set, filtering by Gemini-2.5-flash that meet specific complexity criteria. The actor and critic are updated using a hybrid data mixture:
\begin{enumerate}
    \item Newly generated interaction trajectories from the current curriculum phase.
    \item Historical high-quality data from the replay buffer, specifically filtering for trajectories with perplexity between $1/0.9$ and $1/0.5$.
\end{enumerate}
This dual-source training strategy ensures the agent retains previously learned behaviors while adapting to increasingly complex subgoals.


\subsection{Extra Ablation Studies}
\label{sec:apdx_ab}

\noindent \textbf{The impact of reward shaping ($\alpha$).}
We perform a grid search over $\alpha \in [0.1, 0.8]$ using a held-out validation set of WebArena-Lite tasks. As shown in Table 5, very small shaping factors (0.0–0.1) behave similarly to the sparse-reward baseline, yielding only 30.9–32.1\% success. In this regime, the auxiliary signal is too weak to overcome reward sparsity and thus behaves largely as noise.

\begin{wrapfigure}{r}{0.45\textwidth}
    \centering
     \captionsetup{justification=centering}
    \captionsetup{font=footnotesize}
    {\footnotesize
    \captionof{table}{Impact of reward shaping factor $\alpha$.}
    }
    \label{tab:alpha_ablation}
    \footnotesize
    \begin{tabular}{lcccccc}
        \toprule
        $\alpha$ & 0.0 & 0.1 & \textbf{0.3} & 0.5 & 0.8 \\
        \midrule
        SR (\%) & 30.9 & 31.5 & \textbf{36.4} & 28.5 & 25.5 \\
        \bottomrule
    \end{tabular}
\end{wrapfigure}

Performance peaks at $\alpha = 0.3$, achieving \textbf{36.4\%}, confirming that moderate shaping effectively accelerates credit assignment without drowning out the terminal reward. However, as $\alpha$ increases (0.5–0.8), performance drops sharply to 28.5–25.5\%. Large shaping factors distort the reward landscape, causing the agent to overfit to auxiliary subgoal progress rather than true task completion. These findings highlight the necessity of calibrated shaping magnitudes.

\vspace{0.5cm}

\begin{table}[h]
\centering
\captionsetup{justification=centering}
\caption{Impact of replay buffer filtering based on perplexity (PPL) and rank score ($1/\mathrm{PPL}$).}
\label{tab:perplexity_impact}
\begin{tabular}{lcccc}
\toprule
\textbf{Rank Score ($1/\mathrm{PPL}$)} & $(0.9, 1.0]$ & $\mathbf{[0.5, 0.9]}$ & $[0.0, 0.5)$ & $[0.0, 1.0]$ \\
\midrule
\textbf{PPL Range} & $[1, \tfrac{1}{0.9})$ & $\mathbf{[\tfrac{1}{0.9}, \tfrac{1}{0.5}]}$ & $(\tfrac{1}{0.5}, \infty)$ & Full Range \\
\midrule
SR (\%)   & 27.9 & \textbf{36.4} & 23.6 & 29.1 \\
\bottomrule
\end{tabular}
\end{table}

\noindent \textbf{The impact of perplexity filtering.}
Table~\ref{tab:perplexity_impact} evaluates how different perplexity bands affect data usefulness. Training exclusively on low-perplexity (highly predictable) samples yields only 27.9\% success, suggesting limited learning value from ``too easy'' trajectories. Conversely, using only the high-perplexity tail produces the worst performance (23.6\%), indicating that noisy or out-of-distribution states destabilize the value critic.

The best performance arises from filtering to the moderate-perplexity range $[1/0.9,1/0.5]$ (i.e., rank score $[0.5,0.9]$), achieving \textbf{36.4\%}. These ``borderline-difficult'' transitions provide the most informative gradients—neither trivial nor erratic—leading to more robust and sample-efficient learning.

\begin{algorithm}[h]
\caption{Training the Potential Critic $P_\psi$}
\label{alg:potential_critic_final}
\begin{algorithmic}[1]
\Require Dataset of trajectories $\mathcal{D}$; goal set $\mathcal{G}=\{g_1,\dots,g_K,g_{\mathrm{final}}\}$
\Ensure Trained potential critic $P_\psi(s,g)$

\State $\mathcal{S} \gets \emptyset$ \Comment{Supervision set for regression}

\For{each trajectory $\tau\in\mathcal{D}$}
    \If{trajectory does \emph{not} reach the final goal} 
        \State \textbf{continue} \Comment{Use only positive traces to ensure correlation with success}
    \EndIf

    \State Extract subgoal-completion vector $\mathbf{z}_t$ for all $t$
    \State Identify key steps $t_1,\dots,t_K$ where cumulative completion increases

    \State Compute continuous progress labels $p_t^\ast$ by linear interpolation:
    \[
    p_t^\ast = (1-\alpha_t)\tfrac{j}{K} + \alpha_t\tfrac{j+1}{K},
    \quad \alpha_t=\frac{t - t_j}{t_{j+1}-t_j}
    \]
    with the final segment anchored to the true trajectory end $T$.

    \For{each timestep $t$}
        \State Construct progress-aware state $s_t = (h_{1:t-1}, o_t, g)$
        \State Add $(s_t, g, p_t^\ast)$ to $\mathcal{S}$
    \EndFor
\EndFor

\Statex
\State \textbf{/* Supervised learning */}
\State Train the potential critic via MSE regression:
\[
\mathcal{L}_P(\psi)=\frac{1}{|\mathcal{S}|}\sum_{(s_t,g,p_t^\ast)\in\mathcal{S}}
\left(P_\psi(s_t,g)-p_t^\ast\right)^2.
\]

\State Return trained $P_\psi$.
\end{algorithmic}
\end{algorithm}

\subsection{Potential Critic Training via Supervised Progress Regression}
\label{apdx:potential_training}

A central challenge in long-horizon web tasks is that the environment provides only a sparse terminal reward—typically a single binary success or failure at the end of a multi-step trajectory. Such extreme sparsity makes it difficult for an RL agent to learn meaningful credit assignment, particularly when intermediate actions influence success in subtle, delayed ways. To mitigate this, MiRA introduces a \emph{potential critic} $P_\psi(s,g)$ that predicts a dense, continuous notion of task progress. Instead of relying on handcrafted shaping potentials, which are notoriously brittle in heterogeneous environments like WebArena, the potential critic is learned directly from data via supervised regression, yielding a shaping landscape that reflects the procedural structure of successful trajectories.

The construction of the regression targets is grounded in the SubGoal Checker, which emits a binary subgoal-completion vector $\mathbf{z}_t$ at each timestep. From this signal, we detect the key timesteps $t_1, t_2, \ldots, t_K$ at which cumulative progress increases. Between any pair of these key steps, we interpolate the progress value linearly so that the label $p_t^\ast$ increases smoothly from $j/K$ to $(j{+}1)/K$. This interpolation converts the discrete, staircase-like profile of subgoal completions into a continuous supervision signal defined over the entire trajectory. For the final segment, we intentionally anchor the interpolation interval to the true end of the trajectory rather than the moment the final subgoal is detected; this ensures that the learned potential continues to rise during late-stage ``cleanup'' actions such as verification, scrolling, and form submission, which contribute meaningfully to task resolution despite not triggering additional subgoal boundaries.

An important modeling decision is to fit the potential critic exclusively on \emph{positive} trajectories—those that ultimately reach the final goal. This ensures that the progress targets are always aligned with successful procedural structure. Using failed trajectories as supervision would dilute this relationship, as these episodes often terminate prematurely or contain exploratory detours that do not correspond to genuine forward progress. Restricting the regression dataset to successful episodes produces a potential function that implicitly captures the common semantic subsequences shared across successful attempts, creating a shaping landscape that rises precisely along paths known to lead to final success.

Once progress labels $p_t^\ast$ have been computed, the potential critic is trained to regress these values using a standard MSE objective:
\[
\mathcal{L}_P(\psi)
= \mathbb{E}_{(s_t,g,p_t^\ast)} \bigl[\|P_\psi(s_t,g) - p_t^\ast\|^2\bigr].
\]
The critic is implemented by augmenting a pretrained LLM with a regression head, enabling it to leverage both the semantic structure of the goal instruction and the action–observation history of the trajectory. Because the target labels are smooth and monotonic, optimization is stable, and the learned potential landscape provides a faithful surrogate of distance-to-go long before RL finetuning begins.

During reinforcement learning, the potential critic contributes an auxiliary shaping reward of the form
\[
r^{(\mathrm{aux})}_t
= \alpha \bigl[ P_\psi(s_{t+1},g) - P_\psi(s_t,g) \bigr],
\]
which is added to the environmental reward when computing advantage targets. This shaping term effectively supplies a local gradient that encourages the agent to continue moving ``forward'' in the learned progress space. Crucially, however, the primary value critic $V_\phi$ remains trained only on the true sparse terminal reward, ensuring the auxiliary signal does not distort the optimal policy. The potential critic thus serves as a non-intrusive guidance mechanism: it accelerates credit assignment and stabilizes learning across long horizons, while the ultimate optimization target remains governed solely by the environment's final success signal.

\noindent \textbf{Trajectory Selection and Quality Filtering.} To ensure the potential critic learns from robust and efficient behaviors, we apply strict quality controls to the gathered rollouts. First, we utilize an Outcome Reward Model (ORM) trained via WebRL~\citep{qi2024webrl} to verify the success of each trajectory collected from both the WebRL (Llama3-8b) and baseline RL agents (initialized from the same Gemma SFT checkpoint as MiRA). The task set remains identical to that used in the multi-phase running on the baseline methods. We then filter the verified successes to eliminate signs of suboptimal agency: we remove trajectories containing repetitive loops (defined as identical actions repeated more than 5 times) and exclude inefficient rollouts by retaining only those with a total length of less than 15 steps. This curation process yields a dataset of positive execution traces that represent more direct and logical paths to task completion, ensuring the learned potential landscape encourages efficiency rather than circuitous wandering. Then the potential critic will get continuely trained with main critic and actor.

In summary, supervised progress regression allows MiRA to construct a dense potential landscape grounded in the structure of successful trajectories. This landscape provides smooth intermediate rewards that fill the gaps left by sparse feedback, enabling efficient learning of long-horizon procedural behaviors that would otherwise be inaccessible to standard reinforcement learning.

\subsection{Derivation of the Optimal Policy in KL-Regularized RL}
\label{sec:apdx_RL_adv}

We start with the objective function for KL-regularized Reinforcement Learning (commonly used in RLHF). The goal is to maximize the expected reward while minimizing the KL-divergence from a reference policy $\pi_{\text{ref}}$:

\begin{equation}
    \max_\pi \mathbb{E}_{\pi} \left[ r(s,a) - \beta \log \frac{\pi(a|s)}{\pi_{\text{ref}}(a|s)} \right]
\end{equation}

The closed-form solution for the optimal policy $\pi^*$ is given by the Boltzmann distribution:

\begin{equation}
    \pi^*(a|s) = \frac{1}{Z(s)} \pi_{\text{ref}}(a|s) \exp\left( \frac{Q^*(s,a)}{\beta} \right)
\end{equation}

where $Z(s)$ is the partition function (normalization constant). In the maximum entropy framework, the optimal value function $V^*(s)$ relates to the partition function as follows:

\begin{equation}
    V^*(s) = \beta \log Z(s) \implies Z(s) = \exp\left( \frac{V^*(s)}{\beta} \right)
\end{equation}

Recall the definition of the Advantage function: $A^*(s,a) = Q^*(s,a) - V^*(s)$. We can substitute $Q^*(s,a) = A^*(s,a) + V^*(s)$ back into the policy equation:

\begin{align*}
    \pi^*(a|s) &= \frac{1}{Z(s)} \pi_{\text{ref}}(a|s) \exp\left( \frac{A^*(s,a) + V^*(s)}{\beta} \right) \\
    &= \frac{1}{Z(s)} \pi_{\text{ref}}(a|s) \exp\left( \frac{A^*(s,a)}{\beta} \right) \exp\left( \frac{V^*(s)}{\beta} \right)
\end{align*}

Substituting $Z(s) = \exp\left( \frac{V^*(s)}{\beta} \right)$, the normalization terms cancel out:

\begin{align*}
    \pi^*(a|s) &= \frac{1}{\exp\left( \frac{V^*(s)}{\beta} \right)} \pi_{\text{ref}}(a|s) \exp\left( \frac{A^*(s,a)}{\beta} \right) \exp\left( \frac{V^*(s)}{\beta} \right) \\
    \pi^*(a|s) &= \pi_{\text{ref}}(a|s) \exp\left( \frac{1}{\beta} A^*(s,a) \right)
\end{align*}

\subsection{KL Divergence Formulation for policy update}
\label{sec:apdx_mse_vs_kl}

Many prior studies use KL divergence to measure the distance between two policies. When KL divergence is used to optimize $\pi_\theta$ toward $\pi^*$, the optimization goal becomes:
\begin{align}
& \arg\min_{\theta} \, \mathbb{E}_{s \sim d(s)} \left[ D_{\mathrm{KL}}\left( \pi^*(\cdot|s,I) \,\|\, \pi_\theta(\cdot|s,I) \right) \right] \nonumber \\[6pt]
&= \arg\max_{\theta} \, \mathbb{E}_{s \sim d(s)} \mathbb{E}_{a \sim \pi^*(a|s,I)} \left[ \log \pi_\theta(a|s,I) \right] \label{eq:kl_step1} \\[6pt]
&= \arg\max_{\theta} \, \mathbb{E}_{s \sim d(s)} \int_a \pi_{\mathrm{ref}}(a|s,I) \exp\!\left( \frac{1}{\beta} A^*(s,a,I) \right) \log \pi_\theta(a|s,I) \, \mathrm{d}a \label{eq:kl_step2} \\[6pt]
&= \arg\max_{\theta} \, \mathbb{E}_{s \sim d(s)} \mathbb{E}_{a \sim \pi_{\mathrm{ref}}(a|s,I)} \left[ \log \pi_\theta(a|s,I) \cdot \exp\!\left( \frac{1}{\beta} A^*(s,a,I) \right) \right], \label{eq:kl_final}
\end{align}
where $d(s)$ is a distribution over states, and we have substituted the closed-form optimal policy $\pi^*(a|s,I) = \pi_{\mathrm{ref}}(a|s,I) \exp\!\left( \frac{1}{\beta} A^*(s,a,I) \right)$ from Eq.~\ref{eq:optimal_policy_exp_advantage}.

\end{document}